%% file: main.tex
\documentclass[twoside]{article}
\usepackage{amssymb}
\usepackage{algorithm}
\usepackage{algpseudocode}
\usepackage[preprint]{aistats2026}
\usepackage{tikz}  
\usepackage{hyperref}
\usepackage{todonotes}
\usepackage{subcaption}
\usepackage{float}
\usepackage{stfloats}
\usepackage{tabularx} 
\usepackage{booktabs}
\usepackage{siunitx} 
\usepackage[round]{natbib}

\newtheorem{theorem}{Theorem}
\newtheorem{lemma}[theorem]{Lemma} 
\newtheorem{proposition}[theorem]{Proposition} 
\newtheorem{remark}[theorem]{Remark}

\newtheorem{definition}[theorem]{Definition}

\usepackage{amsfonts}
\newcommand{\BlackBox}{\rule{1.5ex}{1.5ex}}  
\ifdefined\proof
    \renewenvironment{proof}{\par\noindent{\bf Proof\ }}{\hfill\BlackBox\\[2mm]}
\else
    \newenvironment{proof}{\par\noindent{\bf Proof\ }}{\hfill\BlackBox\\[2mm]}
\fi
\usepackage{amssymb} 
\usepackage{mathtools}
 \usepackage{bbm} 
\RequirePackage{pgfmath}
\usepackage{url}
\usepackage{needspace}
\usepackage{etoc}
\begin{document}

%

%

\twocolumn[
\aistatstitle{Minimax-Optimal Two-Sample Test with Sliced Wasserstein}

\aistatsauthor{ Binh Thuan Tran \And Nicolas Schreuder}

\aistatsaddress{ LAMA, LIGM, Université Gustave Eiffel\\
binh-thuan.tran3@univ-eiffel.fr
\And  CNRS, LIGM\\
nicolas.schreuder@cnrs.fr}]

\begin{abstract}
We study the problem of nonparametric two-sample testing using the sliced Wasserstein (SW) distance. While prior theoretical and empirical work indicates that the SW distance offers a promising balance between strong statistical guarantees and computational efficiency, its theoretical foundations for hypothesis testing remain limited. We address this gap by proposing a permutation-based SW test and analyzing its performance. The test inherits finite-sample Type I error control from the permutation principle. Moreover, we establish non-asymptotic power bounds and show that the procedure achieves the minimax separation rate $n^{-1/2}$ over multinomial and bounded-support alternatives, matching the optimal guarantees of kernel-based tests while building on the geometric foundations of Wasserstein distances. Our analysis further quantifies the trade-off between the number of projections and statistical power. Finally, numerical experiments demonstrate that the test combines finite-sample validity with competitive power and scalability, and---unlike kernel-based tests, which require careful kernel tuning---it performs consistently well across all scenarios we consider.
\end{abstract}
\section{INTRODUCTION}\label{sec:intro}
\input{introduction}

\section{BACKGROUND}
\input{background}

\section{RELATED WORK}\label{sec:related}
\input{related_works}

\section{PERMUTATION SW TEST}\label{sec:permutationswtest}
\input{method}

\section{THEORETICAL ANALYSIS}\label{sec:theory}
\input{theory}

\section{NUMERICAL EXPERIMENTS}\label{sec:numerical_experiments}
\input{experiments}

\section{Conclusion}
\input{conclusion}

\subsubsection*{Acknowledgements}
The authors would like to thank Pierre Bizeul for pointing out an interesting reference and Mohamed Hebiri for helpful discussions.

\bibliography{references}

\input{supplement.tex}

\end{document}

%% file: introduction.tex
As machine-generated data becomes ubiquitous---from synthetic images to large language model outputs and scientific simulations---it is necessary to assess whether generated samples are statistically indistinguishable from real data samples. This question fits naturally into the framework of two-sample testing: one sample is drawn from a data distribution $\mu$ and the other from a generative model distribution $\nu$. Formally, the goal is to test
\begin{align*}
    \mathcal{H}_0: \mu = \nu \quad \text{against} \quad \mathcal{H}_1: \mu \neq \nu.
\end{align*}

Two-sample testing is a fundamental problem in statistics, with a long history and a wide range of practical applications \citep{lehmann2005testing}. It has played a central role in areas such as clinical laboratory science \citep{miles2004comparison,zhang2021two}, finance \citep{horvath2014estimation}, bioinformatics \citep{borgwardt2006integrating}, neuroscience \citep{stelzer2013statistical}, and video content analysis \citep{liu2018classifier}. Recently, it has attracted growing interest in emerging machine learning and applied domains, including generative modeling \citep{li2017mmd}, model equality testing for large language models \citep{gao2024model}, and physics \citep{chakravarti2023model}.

Two primary classes of two-sample tests have been developed: \emph{parametric} and \emph{nonparametric}. Parametric tests, such as the classical $t$-test \citep{student1908probable} and Hotelling's two-sample $T^2$ test \citep{hotelling1931generalization}, rely on strong distributional assumptions (e.g., Gaussianity) and may lose power when these assumptions fail. Nonparametric tests avoid such assumptions and are thus more broadly applicable.

Among nonparametric methods, kernel-based tests have become especially prominent. In particular, the kernel two-sample test of \citet{gretton2012mmd}, based on the Maximum Mean Discrepancy (MMD)--a special case of Integral Probability Metrics \citep{muller1997integral}--is widely used for its flexibility and computational tractability. However, kernel-based tests depend critically on kernel choice and hyperparameter tuning, which ideally should reflect the geometry of the underlying data \citep{schrab2023mmd}. 
This limitation motivates exploring alternatives such as optimal transport (OT), which incorporates geometric structure through the choice of ground cost. The Wasserstein distance, in particular, is a natural candidate for two-sample testing, as it captures geometric discrepancies between distributions \citep{villani2008optimal}. 
However, the standard plug-in estimator suffers from the curse of dimensionality, with sharp error rates of order $n^{-1/d}$ \citep{fournier2015rate,weed2019sharp}.

To mitigate this curse of dimensionality, we focus on the sliced Wasserstein distance \citep{rabin2011sw}, which projects measures onto one-dimensional subspaces and attains a parametric $n^{-1/2}$ convergence rate while retaining geometric interpretability \citep{nadjahi2021sliced, nietert2022statistical}. Despite its widespread use in generative modeling and learning, as well as promising empirical performance \citep{grossi2025refereeing}, theoretical guarantees for its application in two-sample testing remain largely unexplored. In particular, finite-sample validity, non-asymptotic power guarantees, and minimax optimality results for tests based on the sliced Wasserstein distances are currently lacking in the literature. This gap limits the theoretical understanding and rigorous adoption of sliced Wasserstein-based two-sample tests in practice.

In this paper, we address these gaps by proposing a permutation-based two-sample test using the sliced Wasserstein distance. We prove that this test enjoys finite-sample Type I error control. Furthermore, we establish non-asymptotic lower bounds on the test's power and demonstrate its minimax optimality against multinomial and bounded-support alternatives. We also analyze the computational–statistical trade-off inherent in the test, showing how the number of projections affects both runtime and statistical performance. A key technical contribution is our analysis of the SW distance between permuted distributions: to the best of our knowledge, this is the first work to control permutation quantiles for Wasserstein-based statistics. This analysis underpins the finite-sample guarantees for the SW test and extends naturally to the general Wasserstein distance. Finally, we complement our theoretical analysis with experiments, showing that our test achieves strong empirical performance and, unlike kernel-based methods whose power depends heavily on kernel choice, performs consistently well across diverse benchmarks without the need for parameter tuning or kernel aggregation.

The remainder of the paper is organized as follows. Section 2 reviews the necessary background on two-sample testing and sliced Wasserstein distance. Section 3 introduces the permutation-based sliced Wasserstein two-sample test. Section 4 presents our theoretical results on finite-sample validity, power analysis, and minimax optimality. Section 5 reports experimental evaluations on synthetic and real datasets. We conclude in Section 6 with a discussion of future research directions. All proofs are deferred to the Appendix.

%% file: background.tex
In this section, we present the two-sample testing problem setup, define the sliced Wasserstein distance, and introduce minimax optimality for testing.

\subsection{Problem setup}

We consider two independent samples drawn from distributions $\mu$ and $\nu$ on $\mathbb{R}^d$:
\[
\mathcal{Y}_n = (Y_1, \dots, Y_n) \overset{\text{i.i.d.}}{\sim} \mu, \quad \mathcal{Z}_m = (Z_1, \dots, Z_m) \overset{\text{i.i.d.}}{\sim} \nu,
\]
both supported on a common ball $ B_D = \{x \in \mathbb{R}^d : \|x\| \leq D\}$ for some $D > 0$, and assume $n \leq m$.

The goal is to test the hypotheses
\[
H_0: \mu = \nu \quad \text{versus} \quad H_1: d(\mu, \nu) \geq \epsilon,
\]
for some discrepancy metric $d$ and fixed $\epsilon > 0$. Many existing tests fit this framework by choosing $d$ as a specific probability discrepancy, such as the total variation distance or the MMD. Later, we will focus on the sliced Wasserstein distance.

A test is a measurable function
\[
\phi_{n,m} : (\mathbb{R}^{d})^n \times (\mathbb{R}^{d})^m \to \{0,1\},
\]
where $\phi_{n,m} = 1$ indicates rejection of the hypothesis $H_0$. Typically, $\phi_{n,m}$ is based on a test statistic $\Psi$ measuring discrepancy between samples, rejecting $H_0$ if $\Psi$ exceeds a threshold.

A test’s performance is characterized by its Type~I (false positive) and Type~II (false negative) errors, defined as follows. Let $\varepsilon>0$ and define

\[
\mathcal{P}_1(\varepsilon) \coloneqq \{(\mu,\nu) : d(\mu,\nu) \geq \epsilon\}.
\]

A level-$\alpha$ test controls the Type I error uniformly by $\alpha \in (0, 1)$, and has power at least $1-\beta$ if its Type II error is bounded by $\beta \in (0, 1)$ uniformly over $\mathcal{P}_1(\varepsilon)$:
\[
\sup_{\mu = \nu} \mathbb{P}_{\mu \times \nu}(\phi_{n,m}{=}1) \leq \alpha, \sup_{(\mu,\nu) \in \mathcal{P}_1(\varepsilon)} \mathbb{P}_{\mu \times \nu}(\phi_{n,m}{=}0) \leq \beta,
\]
where probabilities are taken over the samples and any additional randomness (e.g., random projections).

In Section~\ref{sec:theory}, we establish that our test is of level $\alpha$ and characterize the minimal separation between distributions detectable with power at least $1-\beta$.

\subsection{Sliced Wasserstein Distance}\label{subsec:SlicedWasserStein}

For $p\geq 1$, the $p$-Wasserstein distance~\citep[Section~6]{villani2008optimal} between probability measures $\mu, \nu$ in the space of probability measures with finite $p$-th moments $\mathcal{P}_p(\mathbb{R}^d)$, is defined as
\begin{align*}
\operatorname{W}_p(\mu,\nu) := \left( \inf_{\pi \in \Gamma(\mu,\nu)} \int_{\mathbb{R}^d \times \mathbb{R}^d} \|x - y\|^p \, d\pi(x,y) \right)^{1/p},\end{align*}
where $\Gamma(\mu,\nu)$ is the set of couplings with marginals $\mu$ and $\nu$. The $p$-Wasserstein distance $\operatorname{W}_p$ is a natural candidate for two-sample testing: it vanishes if and only if the distributions coincide and captures geometric discrepancies. 
However, its minimax separation rate scales as $n^{-c/d}$ for some $c>1$, making it impractical in moderate to high dimensions \citep[Section 2.5]{ba2011sublinear,chewi2024statistical}.

To address this limitation, we use the sliced Wasserstein (SW) distance \citep{rabin2011sw}, defined next.
\begin{definition}[SW distance]\label{slicedwassersteindistance}
For $p \geq 1$ and $\mu, \nu \in \mathcal{P}_p(\mathbb{R}^d)$, the sliced Wasserstein distance is
\begin{equation}\label{definitionSlicedWasserstein}
\operatorname{SW}_p(\mu, \nu) := \left( \int_{\mathbb{S}^{d-1}} \operatorname{W}_p^p\big( \Pi^\theta_{\#}\mu, \Pi^\theta_{\#}\nu \big) \, \sigma(d\theta) \right)^{1/p},
\end{equation}
where $\Pi^\theta(x){=}\langle \theta, x \rangle$ is the projection onto direction $\theta$ on the unit sphere $\mathbb{S}^{d-1}$, and $\sigma$ is the uniform measure on the unit sphere.
\end{definition}
By averaging one-dimensional Wasserstein distances over random projections, $\operatorname{SW}_p$ benefits from a dimension-free sample complexity rate  while preserving geometric interpretability and has been successfully applied in large-scale statistical problems (see Sections~\ref{sec:intro} and \ref{sec:related}). In order to estimate the SW distance, we replace the unknown measures $\mu$ and $\nu$ by their empirical counterpart
\[
\widehat{\mu}_n\coloneqq \frac{1}{n} \sum_{i=1}^n \delta_{Y_i} \quad \text{ and } \quad \widehat{\nu}_m \coloneqq\frac{1}{m} \sum_{i=1}^m \delta_{Z_i}.
\]
Moreover, we approximate the expectation with respect to the uniform distribution on the sphere $\sigma$ by Monte Carlo sampling: we draw $L \geq 1$ independent directions $\Theta = (\theta_1, \dots, \theta_L)$ from $\sigma$ and compute the tractable statistic
\begin{align}\label{teststatistic}
\widehat{\operatorname{SW}}^p_p(\widehat{\mu}_n, \widehat{\nu}_m) \coloneqq \frac{1}{L} \sum_{\ell=1}^L \operatorname{W}_p^p\big( \Pi^{\theta_\ell}_{\#}\widehat{\mu}_n, \Pi^{\theta_\ell}_{\#}\widehat{\nu}_m \big).
\end{align}

For brevity, we write $\widehat{\operatorname{SW}}^p_p$ for the estimator in \eqref{teststatistic}. 

\subsection{Minimax optimality}\label{sec:minimaxoptimality}

Minimax theory provides a benchmark for testing by characterizing the smallest detectable difference for level-$\alpha$ tests \citep{ingster1993asymptotically,baraud2002non}. Define
\[
\Phi_{n+m,\alpha} \coloneqq \left\{ \phi_{n,m} : \sup_{\mu = \nu} \mathbb{P}_{\mu \times \nu}(\phi_{n,m} = 1) \leq \alpha \right\}
\]
as the set of tests controlling Type I error at level $\alpha$ based on $n$ samples from $\mu$ and $m$ samples from $\nu$.
Given $\epsilon>0$, the minimax risk is then defined as the minimal worst-case Type II error over $\mathcal{P}_1(\epsilon)$ among all level-$\alpha$ tests,
\[
R^{\dagger}_{n+m, \epsilon} \coloneqq \inf_{\phi_{n,m} \in \Phi_{n+m,\alpha}} \sup_{(\mu, \nu) \in \mathcal{P}_1(\epsilon)} \mathbb{P}_{\mu \times \nu}(\phi_{n,m} = 0).
\]
The minimax separation $\epsilon_{n,m}^\dagger$ is the smallest discrepancy $\epsilon>0$ such that a level-$\alpha$ test exists with Type II error at most $\beta$:
\begin{equation}\label{eq:minimax_separation_def}
\epsilon_{n,m}^\dagger\coloneqq \inf \left\{ \epsilon > 0 : R^{\dagger}_{n+m, \epsilon} \leq \beta \right\},
\end{equation}
for fixed $\beta \in (0,1-\alpha)$. Intuitively, $\epsilon_{n,m}^\dagger$ characterizes the detection boundary below which no $\alpha$-level test can reliably distinguish alternatives from the null.

For a given separation $\epsilon>0$, the maximum Type II error of a test $\phi_{n,m} \in \Phi_{n+m,\alpha}$ over the class of alternative distributions $\mathbb{P}_1(\epsilon)$ is
\[
R_{n+m,\epsilon}(\phi_{n,m}) \coloneqq \sup_{(\mu, \nu) \in \mathcal{P}_1(\epsilon)} \mathbb{P}_{\mu \times \nu}(\phi_{n,m} = 0).
\]

Then, its minimum separation is
\[
\tilde{\epsilon}_{\phi,n,m}\coloneqq \inf \left\{\epsilon > 0 : R_{n+m,\epsilon}(\phi_{n,m}) \leq \beta \right\}.
\]

A test $\phi_{n,m}$ is minimax rate-optimal if its minimum separation $\tilde{\epsilon}_{\phi,n,m}$ is equivalent to
$\epsilon_{n,m}^\dagger$ up to 
constant factors. We show in Section~\ref{sec:theory} that the sliced Wasserstein test we consider is minimax rate-optimal.

%% file: related_works.tex
The problem of two-sample testing has been extensively studied for over a century, with two main classes of approaches: \emph{parametric} and \emph{non-parametric}.

\paragraph{Classical parametric tests.}
Early approaches to the two-sample problem include the Kolmogorov–Smirnov test \citep{kolmogorov1933sulla,smirnov1948table}, Student’s $t$-test for mean comparison \citep{student1908probable}, and Hotelling’s $T^2$ test as a multivariate generalization of the $t$-test \citep{hotelling1931generalization}. These procedures are computationally simple and well understood, but they are either restricted to the univariate setting (e.g., the KS test) or to testing specific moments such as the mean (e.g., $t$-test, Hotelling’s $T^2$), which limits their applicability in high-dimensional or nonparametric problems.

\paragraph{Kernel-based methods.}
Kernel two-sample testing has emerged as a popular nonparametric alternative to classical parametric tests, designed to handle complex and high-dimensional data.
The maximum mean discrepancy (MMD) test of \citet{gretton2012mmd} is widely used for its tractability and power, and it is closely related to energy distance \citep{szekely2004testing} through the equivalence established by \citet{sejdinovic2013equivalence}. Because test performance is sensitive to kernel choice, later work has proposed kernel aggregation for adaptivity \citep{biggs2023mmd, schrab2023mmd}. In parallel, scalable variants have been developed for large-scale problems, including random Fourier feature approximations \citep{choi2024computational, mukherjee2025minimax}, Nyström subsampling \citep{chatalic2025efficient}, and coreset-based methods \citep{pmlr-v206-domingo-enrich23a}. While conceptually related to OT-based approaches—since both compare distributions via a discrepancy measure—kernel tests are grounded in reproducing kernel Hilbert space embeddings, whereas OT-based methods exploit the geometry induced by transport costs.

\paragraph{Optimal transport-based methods.}
Parallel to kernel testing, there has been growing interest in tests based on optimal transport. \citet{ramdas2017wasserstein} survey connections between Wasserstein distance, MMD, and energy distance, framing entropic OT as interpolating between transport and kernel discrepancies. Building on this, several works analyze the statistical properties of Wasserstein-based tests \citep{imaizumi2022wasserstein, gonzalez2023torus} and propose projection-based approaches such as projected Wasserstein tests \citep{wang2021linear,wang2021kernel}. Sliced variants, which exploit the tractability of one-dimensional OT, have been studied extensively: the sliced Wasserstein distance was introduced by \citet{rabin2011sw}, with theoretical properties developed by \citet{bonnotte2013thesis}. Generalizations include generalized and max-sliced Wasserstein distances \citep{kolouri2019gsw,deshpande2019maxsw}, subspace-robust Wasserstein \citep{paty2019subspace}, and recent work on statistical guarantees for max-sliced variants \citep{boedihardjo2024rates,wang2024kernelsw}. Closest to our work, \citet{hu2025maxsliced} propose a max-sliced Wasserstein test with bootstrap calibration and asymptotic validity, while finite-sample guarantees remain an open question.

\paragraph{Project–then–test methods.}
A related line of work uses random projections to reduce dimensionality before applying classical test statistics. For example, \citet{lopes2011random} propose projecting high-dimensional data onto random directions and then applying Hotelling’s $T^2$ test.  

\paragraph{Calibration and optimality.}
Permutation testing plays a central role in our approach. Its finite-sample validity has been long established \citep{hoeffding1952,hemerik2018permutation}, and more recent results show that permutation tests can achieve the minimax optimal detection boundary \citep{kim2022minimax}. This positions permutation calibration as a natural complement to Wasserstein-based discrepancies, allowing us to combine strong finite-sample guarantees with minimax optimality.

In summary, classical tests are restricted in scope, kernel tests provide flexible and efficient nonparametric alternatives, and OT-based tests capture geometric aspects of the distributions but often lack finite-sample guarantees. We propose a permutation-based sliced Wasserstein test that combines size control with non-asymptotic power bounds, achieves the minimax rate $n^{-1/2}$, and quantifies the trade-off between computational efficiency and statistical power.

%% file: method.tex
Before presenting our method in Algorithm~\ref{algorithm}, we detail the permutation procedure and the obtained test.

\subsection{Permutation procedure}\label{permutation approach}
Recall $\mathcal{Y}_n = (Y_1, \dots, Y_n)$ and $\mathcal{Z}_m = (Z_1, \dots, Z_m)$ denote the two samples, and let $\Theta = (\theta_1, \dots, \theta_{L})$ be the set of projection directions (drawn i.i.d. uniformly from the sphere). We form the pooled dataset $\mathcal{X}_{n+m} = (X_1, \dots, X_{n+m})$ as
\begin{align*}
X_i \coloneqq Y_i \quad \text{for } 1 \leq i \leq n, \quad X_{n+i} \coloneqq Z_i \quad \text{for } 1 \leq i \leq m.
\end{align*}
For a permutation $\pi \in S_{n+m}$, the symmetric group on $\{1, \dots, n+m\}$, define the permuted dataset $\mathcal{X}^\pi_{n+m} = (X_{\pi(1)}, \dots, X_{\pi(n+m)})$. Construct empirical measures
\[
\widehat{\mu}^\pi_n \coloneqq \frac{1}{n} \sum_{i=1}^n \delta_{X_{\pi(i)}}, \quad \widehat{\nu}^\pi_m \coloneqq \frac{1}{m} \sum_{i=n+1}^{n+m} \delta_{X_{\pi(i)}},
\]
and compute the empirical sliced Wasserstein distance
\begin{equation}\label{eq:empericalpermu}
\widehat{\operatorname{SW}}^{p,\pi}_p\coloneqq \widehat{\operatorname{SW}}^{p}_p(\widehat{\mu}^\pi_n, \widehat{\nu}^\pi_m).
\end{equation}

Let $N \coloneqq n + m$, and denote by $F_{\widehat{\operatorname{SW}}^{p,\pi}_p}$ the permutation empirical cumulative distribution function:
\[
F_{\widehat{\operatorname{SW}}^{p,\pi}_p}(t)\coloneqq \frac{1}{|S_N|} \sum_{\pi \in S_N} \mathbf{1}\{\widehat{\operatorname{SW}}^{p,\pi}_p \leq t\},
\]
where $S_N$ is the set of all permutations of $\{1, \dots, N\}$. We write the $1-\alpha$ quantile of this distribution as
\begin{align}\label{eq4}
c_{1-\alpha,N} \coloneqq \inf \{t : F_{\widehat{\operatorname{SW}}^{p,\pi}_p}(t) \geq 1 - \alpha\}.
\end{align}

\paragraph{Decision making.} Given the quantile $c_{1-\alpha,N}$, we reject the null hypothesis if $\widehat{\operatorname{SW}}^{p}_p > c_{1-\alpha,N}$. This critical value ensures finite-sample Type I error control under the permutation-invariance (exchangeability) assumption, which holds for both two-sample and independence testing problems (see Section~\ref{sec:controltypeI}).

\begin{remark}\label{remark:MonteCarlopermutation}
Exact computation of the critical value \eqref{eq4} is generally infeasible for large samples, so it is commonly approximated via Monte Carlo simulations. This approximation can be made arbitrarily accurate by increasing the number of sampled permutations (see Appendix or \citet[Lem. 6]{pmlr-v206-domingo-enrich23a}).
\end{remark}

Let $r$ be the uniform distribution on $S_N$ and let $\mathbb{Z}_B \coloneqq (\pi_b)_{1 \le b \le B}$ be a collection of  i.i.d. samples from $r$. Define $\pi_{B+1} := \text{id}$ as the identity permutation. Following the approach of \citet[Lemma~1]{romano2005exact}, and in order to achieve the prescribed non-asymptotic test level, we then consider the statistics $\widehat{\operatorname{SW}}^{p,\pi_b}_p$ for $b=1,\dots,B+1$ and estimate the $(1-\alpha)$ quantile of the permutation distribution by
\begin{align*}
\widehat{c}^B_{1-\alpha,N} := \inf \left\{ t : \frac{1}{B+1} \sum_{i=1}^{B+1} \mathbf{1}\{\widehat{\operatorname{SW}}^{p,\pi_i}_p \le t\} \ge 1-\alpha \right\}.
\end{align*}
Letting $\widehat{\operatorname{SW}}^{p,r_{\bullet 1}}_p \le \widehat{\operatorname{SW}}^{p,r_{\bullet 2}}_p \le \cdots \le \widehat{\operatorname{SW}}^{p,r_{\bullet (B+1)}}_p$ denote the order statistics of $\{\widehat{\operatorname{SW}}^{p,\pi_b}_p\}_{b=1}^{B+1}$, the quantile estimator equals $\widehat{\operatorname{SW}}^{p,r_{\bullet \lceil (B+1)(1-\alpha) \rceil}}_p$. The Sliced Wasserstein Test is then defined as
\[
\Delta(\mathcal{Y}_n, \mathcal{Z}_m, \mathbb{Z}_B, \Theta) := \mathbf{1}\big(\widehat{\operatorname{SW}}^p_p > \widehat{c}^B_{1-\alpha,N}\big).
\]

Algorithm~\ref{algorithm} summarizes the testing procedure.

\begin{algorithm}[h]
\caption{SW-Permutation Test}\label{algorithm}
\begin{algorithmic}[1]
\Require Datasets $\mathcal{Y}_n = (Y_1, \dots, Y_n)$ and $\mathcal{Z}_m = (Z_1, \dots, Z_m)$; significance level $\alpha \in (0,1)$; number of permutations $B$; number of projection $L$.
\Ensure Decision $\Delta \in \{0,1\}$
\State Form pooled dataset $\mathcal{X}_{n+m} = (X_1, \dots, X_{n+m})$ as
\[
X_i := 
\begin{cases}
Y_i & 1 \leq i \leq n, \\
Z_{i - n} & n < i \leq n + m
\end{cases}
\]
\State Sample $\theta_1, \dots, \theta_L \overset{\text{i.i.d.}}{\sim} \mathrm{Unif}(\mathbb{S}^{d-1})$
\State Sample $\pi_1, \dots, \pi_B \overset{\text{i.i.d.}}{\sim} \mathrm{Unif}(S_{n+m})$
\State Set $\pi_{B+1} := \text{id}$ (identity permutation)
\For{$b = 1$ to $B+1$}
    \State Define empirical measures
    \[
    \widehat{\mu}_n^{\pi_b} := \frac{1}{n} \sum_{i=1}^n \delta_{X_{\pi_b(i)}}, \quad
    \widehat{\nu}_m^{\pi_b} := \frac{1}{m} \sum_{i=n+1}^{n+m} \delta_{X_{\pi_b(i)}}
    \]
    \State Compute statistic $\widehat{\operatorname{SW}}_p^{p,\pi_b} := \widehat{\operatorname{SW}}_p^p(\widehat{\mu}_n^{\pi_b}, \widehat{\nu}_m^{\pi_b})$
\EndFor
\vspace{0.5em}
\State Set the critical $\widehat{c}^B_{1-\alpha,n+m}$ as the $(1-\alpha)$ empirical quantile of $\{\widehat{\operatorname{SW}}_p^{p,\pi_b}\}_{b=1}^{B+1}$
\vspace{0.5em}
\State \Return $\Delta \coloneqq\mathbf{1}\left( \widehat{\operatorname{SW}}_p^{p,\pi_{B+1}} > \widehat{c}^B_{1-\alpha,n+m} \right)$
\end{algorithmic}
\end{algorithm}
\begin{remark} 
While bootstrap or subsampling methods are often used to calibrate critical values (see, e.g., \citet{hu2025maxsliced}), their asymptotic validity rely on well-behaved limiting distributions and do not guarantee uniform finite-sample size control, as required for the finite sample minimax framework that we consider.
\end{remark}

\subsection{Computational complexity}
With a one-time pre-computation of projections $O(LdN)$ and per-projection sorting $O(L N \log(N))$, and the main loop $O(LBN)$, our method achieves an overall time complexity of $O\big(LN(d+\log N+B))$. 
Moreover, storing the original and projected datasets, as well as the permuted statistics requires space complexity $O(Nd + NL + B)$.

Note that, choosing $L = N$ projections leads to quadratic time complexity (up to logarithmic factors) in the sample size, comparable to that of the MMD two-sample test.
Reducing the number of projections can mitigate this cost, trading off between statistical power and computational efficiency. A characterization of this trade-off can be obtained from Theorem~\ref{maintheorem}.

%% file: theory.tex
In this section, we analyze the proposed test by studying its level and power. Without loss of generality, we assume $n \le m$. We let $B \geq 1$ denote the number of permutations and $L \geq 1$ the number of projections. For space constraints and to improve readability, all proofs are deferred to the Appendix and all constants are made explicit in the proof.

\subsection{Level of the test}\label{sec:controltypeI}
We begin by establishing that the proposed test controls the Type I error.

It is now a well-known fact that permutation-based tests control the Type~I error at finite samples for any test statistic when the data are exchangeable under the null hypothesis \citep{hoeffding1952,lehmann2005testing}. This property is crucial since it ensures validity without asymptotic approximations. In the two-sample setting, this holds under $H_0:\mu=\nu$ since all observations are i.i.d. and the joint distribution is invariant under permutations of group labels. We formalize this below.

\begin{theorem}[Type I error Control]\label{controltypeI}
The test $\Delta$ defined in Algorithm~\ref{algorithm} has non-asymptotic level $\alpha$ for any $\alpha \in (0,1)$. That is
\begin{align*}\mathbb{P}_{\mu\times \mu \times r\times\sigma}\left(\Delta\left(\mathcal{Y}_n,\mathcal{Z}_m,\mathbb{Z}_{B},\Theta\right)=1\right)\le \alpha,
\end{align*}
where the probability is taken over the samples, projection directions, and permutation randomness.
\end{theorem}

\subsection{Power of the test}
While Type I error control follows from standard permutation arguments, analyzing the test’s power is more challenging. The main difficulty is that the critical value is defined as a data-dependent permutation quantile of the sliced Wasserstein statistic. For MMD, existing analyses exploit its $U$- or $V$-statistic structure together with associated concentration inequalities \citep{schrab2023mmd,pmlr-v206-domingo-enrich23a}, but the sliced Wasserstein distance does not admit such a representation, preventing the direct use of these techniques. To address this issue, we leverage properties of the sliced Wasserstein distance under permutations, in particular, a permutation bounded-differences (McDiarmid-type) inequality and an optimal matching bound from \citet[Corollary~5]{bobkov2021simple}. These ingredients yield the non-asymptotic power bound stated in Theorem~5.




\begin{theorem}[Power Control]\label{maintheorem}
Let $\beta \in (0,1)$ and $\frac{1}{B+1} \le \alpha < 1$. The test described in Algorithm~\ref{algorithm} has power at least $1-\beta$ provided that
\begin{align*}
	\operatorname{SW}_p^p(\mu,\nu)  \geq \frac{C(D, p, \alpha, \beta, B)}{\sqrt{L \wedge n}},
\end{align*}
where the constant $C(D, p, \alpha, \beta, B)$ is made explicit in the proof.
\end{theorem}

This result shows that the SW test achieves the parametric separation rate $n^{-1/2}$ up to constants. To match the sample size, the number of projections $L$ should scale proportionally to $n$, ensuring that projection variance does not limit power.

\subsection{Minimax optimality}
To assess the statistical optimality of the proposed test, we now study its minimax separation rate. Establishing minimax lower bounds in nonparametric testing typically involves constructing least favorable alternatives and adapting classical arguments, which is the approach we take here. These are the first minimax lower bounds for two-sample testing with respect to the sliced Wasserstein distance. Consider the class
\begin{align*}
\mathcal{P}_{\mathbb{R}^d}(D) := \bigl\{ \mu \in \mathcal{P}(\mathbb{R}^d) : \mathrm{diam}(\mathrm{supp}(\mu)) \leq D \bigr\},
\end{align*}
comprising all probability measures on $\mathbb{R}^d$ with support diameter at most $D>0$. Taking $L \geq cn$ for some absolute constant $c>0$, Theorem~\ref{maintheorem} implies an upper bound on the minimax separation rate (see Eq.~\eqref{eq:minimax_separation_def})
\begin{align*}
\epsilon_{n,m}^\dagger \leq C n^{-1/2}.
\end{align*}

We now establish lower bounds in two scenarios to assess the tightness of the upper bound. First, we provide a lower bound for the simpler case of multinomial distributions. We will then consider more generally distributions with bounded support. The following result is obtained via an adaptation of the classical lower bound technique from \citet{ingster1987minimax}.

\begin{proposition}\label{minimaxmultinomial}
Let $\alpha \in \left(\frac{1}{B+1},1\right)$ and $\beta \in (0,1-\alpha)$. For any integer $d \ge 1$, let $[d]\coloneqq\{1,\dots,d\}$ and denote by $\mathcal{P}^{(d)}_{\mathrm{Multi}}$
the class of multinomial distributions on $[d]$. For the two-sample testing problem over $\mathcal{P}^{(d)}_{\mathrm{Multi}}\times \mathcal{P}^{(d)}_{\mathrm{Multi}}$, the minimax separation rate satisfies
\begin{align*}
\epsilon_{n,m}^\dagger \geq \frac{C(d, \alpha, \beta)}{\sqrt{n}},
\end{align*}
where the constant $C(d, \alpha, \beta)$ is made explicit in the proof.
\end{proposition}
The above result establishes that no test can achieve faster than $n^{-1/2}$ separation in the simple multinomial setting.

Turning to the more general class of distributions with bounded support $\mathcal{P}_{\mathbb{R}^d}(D)$, we obtain the following lower bound using Le Cam’s two-point method~\citep{lecam1973convergence,le2012asymptotic}.

\begin{proposition}\label{lowerboundseperation}
Let $\alpha,\beta \in (0,1)$ with $\alpha + \beta < 0.5$. For the two-sample testing problem over $\mathcal{P}_{\mathbb{R}^d}(D)$, the minimax separation rate satisfies
\begin{align*}
\epsilon_{n,m}^\dagger 
\geq \frac{C(D, \alpha, \beta, p)}{\sqrt{n}}.
\end{align*}
\end{proposition}

Thus, the lower bound extends beyond discrete settings to general bounded-support distributions.

Together with Theorem~\ref{maintheorem}, these results imply that the permutation SW test achieves the minimax optimal separation rate $n^{-1/2}$ over both multinomial and bounded-support distributions, matching known optimal rates for MMD-based tests \citetext{\citealp[Position 4.4]{kim2022minimax}; \citealp[Section E.10.1]{kim2023differentially}}.

%% file: experiments.tex
To complement our theoretical findings, we empirically evaluate the performance of the SW test described in Algorithm~\ref{algorithm}. We focus on the balanced case $n = m$ and assess performance across various numbers of projections $L$. We compare against several baselines: the Projected Wasserstein (PW) test \citep{wang2021linear} and the MMD test \citep{gretton2012mmd}. For MMD, we consider linear, Gaussian, and Laplace kernels, with bandwidths for the latter two selected via the median heuristic \citep{garreau2017large}. For the PW test, we use the authors’ official implementation with default settings (projection dimension $3$ and a $50/50$ split for training and testing), as no practical guidance on these parameters was provided. All experiments were run on a machine with an Intel Core Ultra 9 185H CPU and 64\,GB of RAM.

As discussed in Section~\ref{sec:related}, several scalable variants of MMD have been developed for large-scale problems. However, these approximations do not increase statistical power beyond standard MMD, so we exclude them from our benchmarks and focus on contrasting how SW- and MMD-based tests distinguish distributions. Moreover, we also exclude MMD aggregation methods: while aggregation improves adaptivity and can outperform individual kernels, our goal is to compare discrepancies rather than kernel selection strategies, so we restrict attention to standard kernels.

In all experiments, we set the significance level to $\alpha = 0.05$ and use $B = 200$ permutations. Empirical power is estimated from $150$ independent repetitions, parallelized with Joblib. Error bars represent 95\% confidence intervals computed from the normal approximation of the empirical proportion (mean $\pm 1.96$ standard errors).

\subsection{Tests on Synthetic Datasets}\label{sec:Testsonsynthetic}
Our first experiment evaluates power under a Gaussian covariance shift: samples are drawn from $\mu = \mathcal{N}(0, I_{60})$ and $\nu = \mathcal{N}(0, \Sigma)$, where $\Sigma = \operatorname{diag}(\delta^2, \delta^2, 1, \dots, 1)$ with shift magnitude $\delta = 2.7$.

The second scenario examines geometric differences between distributions: the uniform distribution on the 5D unit sphere versus the uniform distribution on the 5D unit ball. This setting tests performance when one distribution lies on a lower-dimensional manifold embedded within the other’s support.

As shown in Figure~\ref{fig:empiricalresults}, the PW test achieves the highest power under covariance shifts, consistent with \citet{wang2021kernel}, while the SW test performs competitively with Gaussian and Laplace MMD. In the ball–vs.–sphere scenario, the MMD test with a Laplace kernel attains the best performance, but the SW test substantially outperforms both Gaussian MMD and PW, even with a limited number of projection directions. The weaker performance of Gaussian MMD likely stems from over-smoothing under the median heuristic, which blurs sharp radial differences, whereas the Laplace kernel’s slower decay preserves more of this contrast. By contrast, MMD with a linear kernel performs worst across both settings, since with $k(x,y)=x^\top y$ it reduces to a mean-difference test, which vanishes in these cases. Overall, while the SW test is not always the most powerful, it delivers consistently strong performance across benchmarks without requiring parameter tuning, demonstrating robustness to diverse distributional differences.

\begin{figure*}[t]
  \centering
  \begin{subfigure}{0.325\textwidth}
    \centering
    \includegraphics[width=\linewidth]{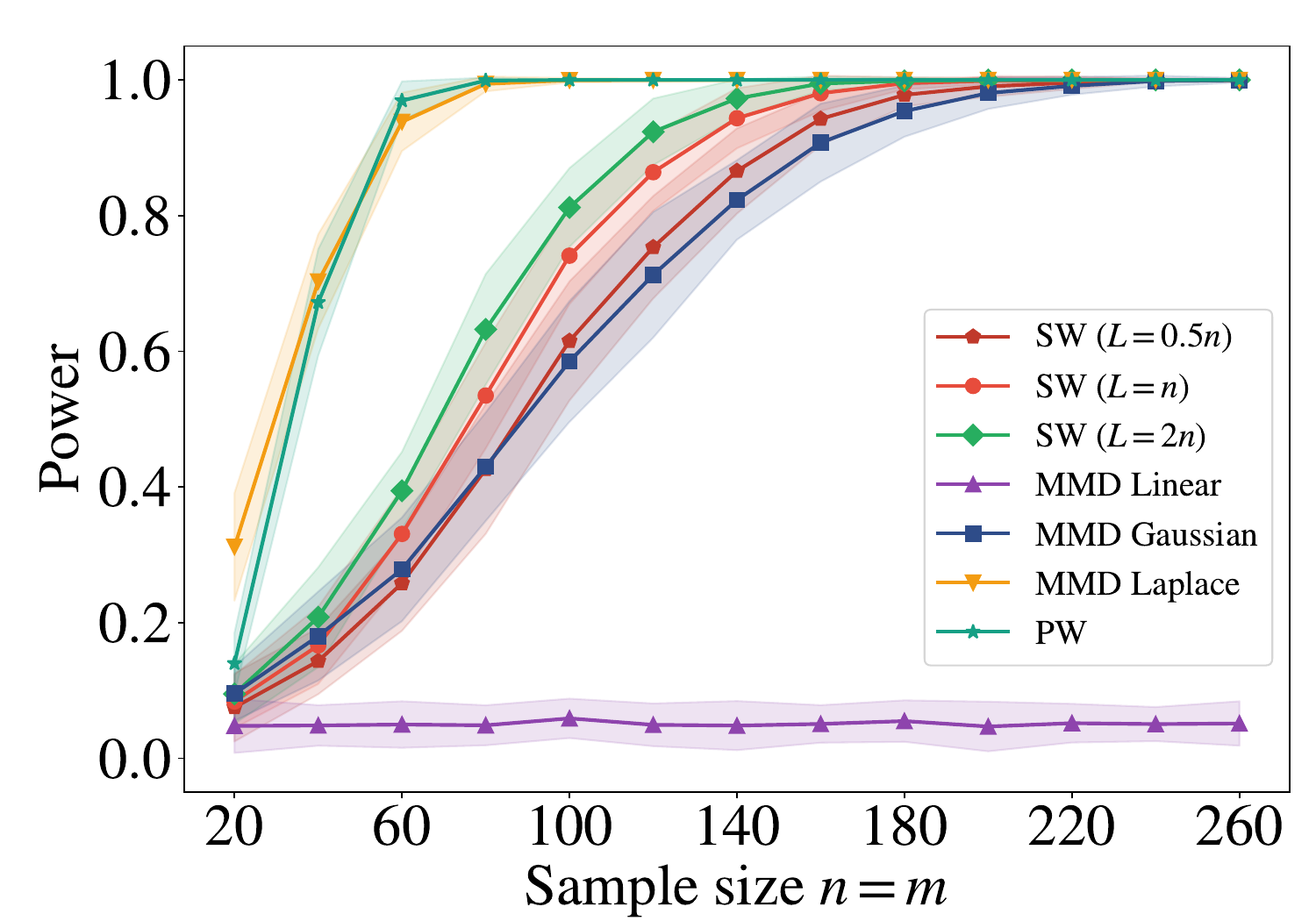}  
  \end{subfigure}\hspace{0.0005\textwidth}
  \begin{subfigure}{0.325\textwidth}
    \centering
    \includegraphics[width=\linewidth]{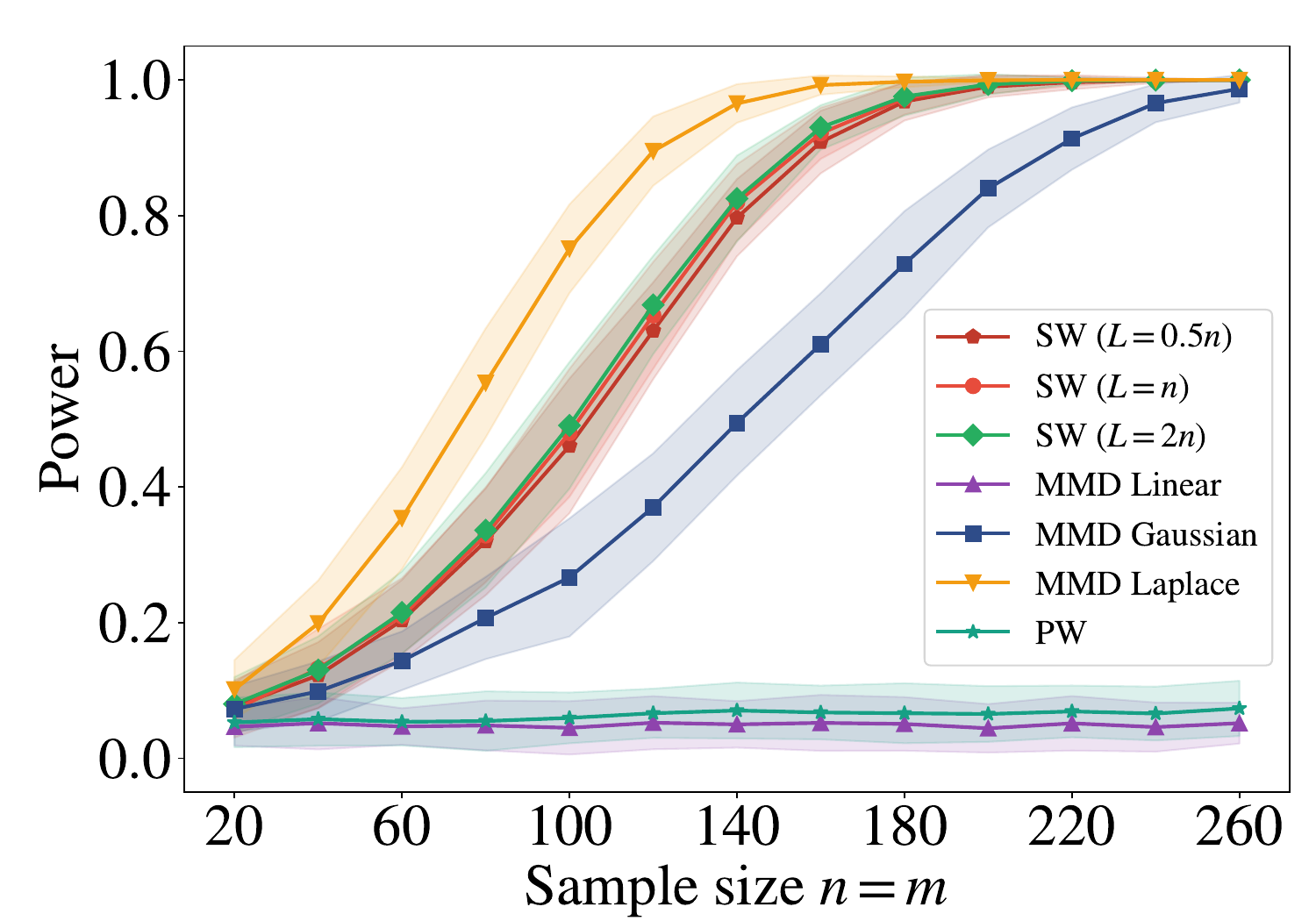}  
  \end{subfigure}\hspace{0.0005\textwidth}
  \begin{subfigure}{0.325\textwidth}
    \centering
    \includegraphics[width=\linewidth]{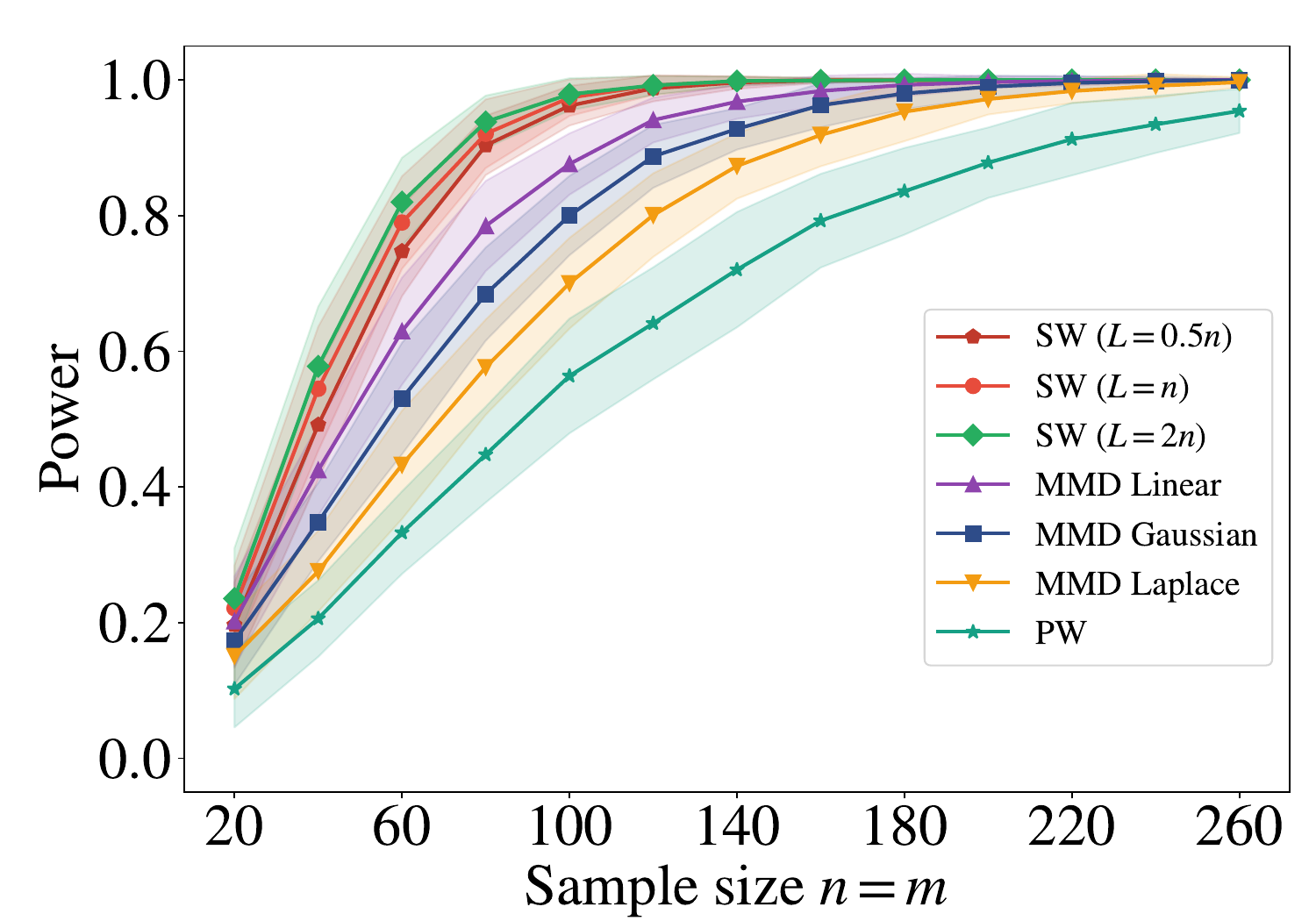}
  \end{subfigure}
  \caption{Power across three scenarios: Gaussian covariance shift, ball vs. sphere, and MNIST mixture.}
  \label{fig:empiricalresults} 
\end{figure*}

\subsection{Tests on MNIST Handwritten Digits}
We apply the two-sample testing procedures to the MNIST dataset \citep{lecun2012mnist}. Let $\mu_6$ denote the uniform distribution over images of the digit $6$ and $\mu_9$ the corresponding distribution for digit $9$. We define the mixture $\nu = 0.85 \mu_6 + 0.15 \mu_9$. Following \citet{wang2021kernel}, we preprocess the dataset by applying a sigmoid transformation to each image so that all pixel values lie within $[0,1]$. The SW test effectively discriminates between $\mu_6$ and $\nu$ by capturing localized intensity differences in 1D projections.

This experiment highlights the SW test’s ability to detect subtle distributional differences in high-dimensional image data without requiring kernel parameter tuning, confirming its practical applicability.

\subsection{Effect of the Number of Projections}\label{Sec:EffectOfTheNumberProjections}
While the SW test generally performs well on our three benchmark datasets, it underperforms relative to MMD tests with Gaussian, Laplace, and linear kernels in an additional Gaussian mean-shift experiment for a range of projection numbers $L \in {0.5n, n, 2n}$ (see Appendix). To better understand this behavior, we examine the effect of increasing the number of projections $L$ on both power and computation time. Fixing $n=m=140$ in the mean-shift setting, Figure~\ref{fig:improvepower} shows that power improves steadily with larger $L$, though at the cost of increased computation time (Figure~\ref{ballsphere_timing}).
\begin{figure}[h]
  \centering
  \includegraphics[width=0.9\linewidth]{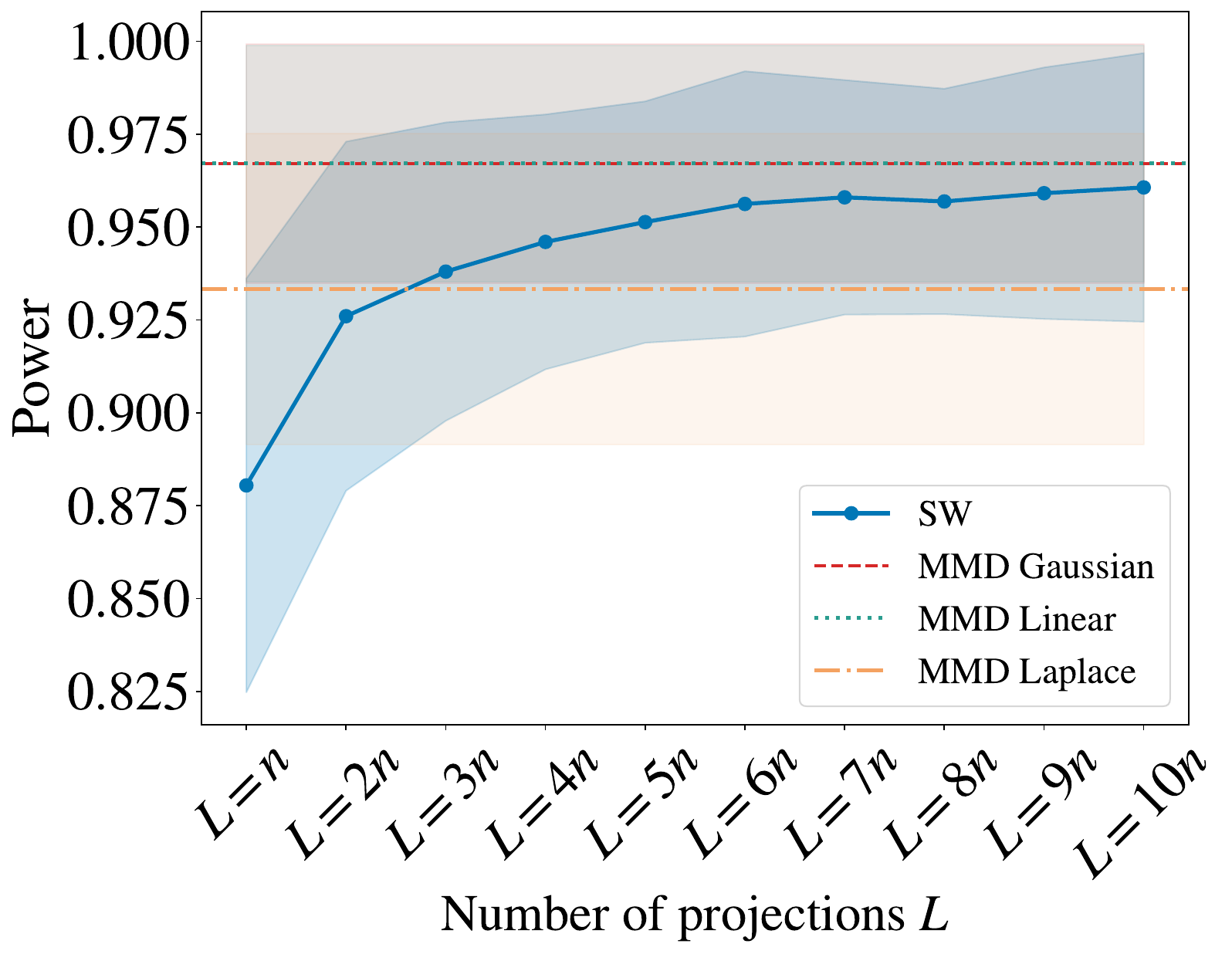}
  \caption{SW Test power vs. number of projections (fixed sample sizes $n=m=140$)}
  \label{fig:improvepower}
\end{figure}

Fortunately, the computational bottleneck of permutation-based SW testing is mitigated by its projection structure: since it computes many independent 1D projections, the workload is naturally amenable to GPU parallelization. This allows scaling the number of projections with only moderate increases in wall-clock time, given sufficient GPU resources. In contrast, MMD computations hinge on the kernel matrix, whose size is fixed once the sample size is given, limiting further scaling through additional parallelism. In practice, we find that SW implementations can better exploit hardware parallelism and thus often scale better than naive MMD implementations under similar resource constraints. 

To validate this, we compared computation times using two sets of $140$ samples from a $60$-dimensional mean-shifted Gaussian distribution. We measured the average execution time over 100 repetitions on a local CPU, Google Colab’s T4, and A100 GPUs. Results in Figure~\ref{fig:improvetime} confirm substantial speedups for the SW test on GPUs, though larger-scale studies would be valuable to fully assess its scalability in practice.

\begin{figure}[t]
  \centering
  \includegraphics[width=0.95\linewidth]{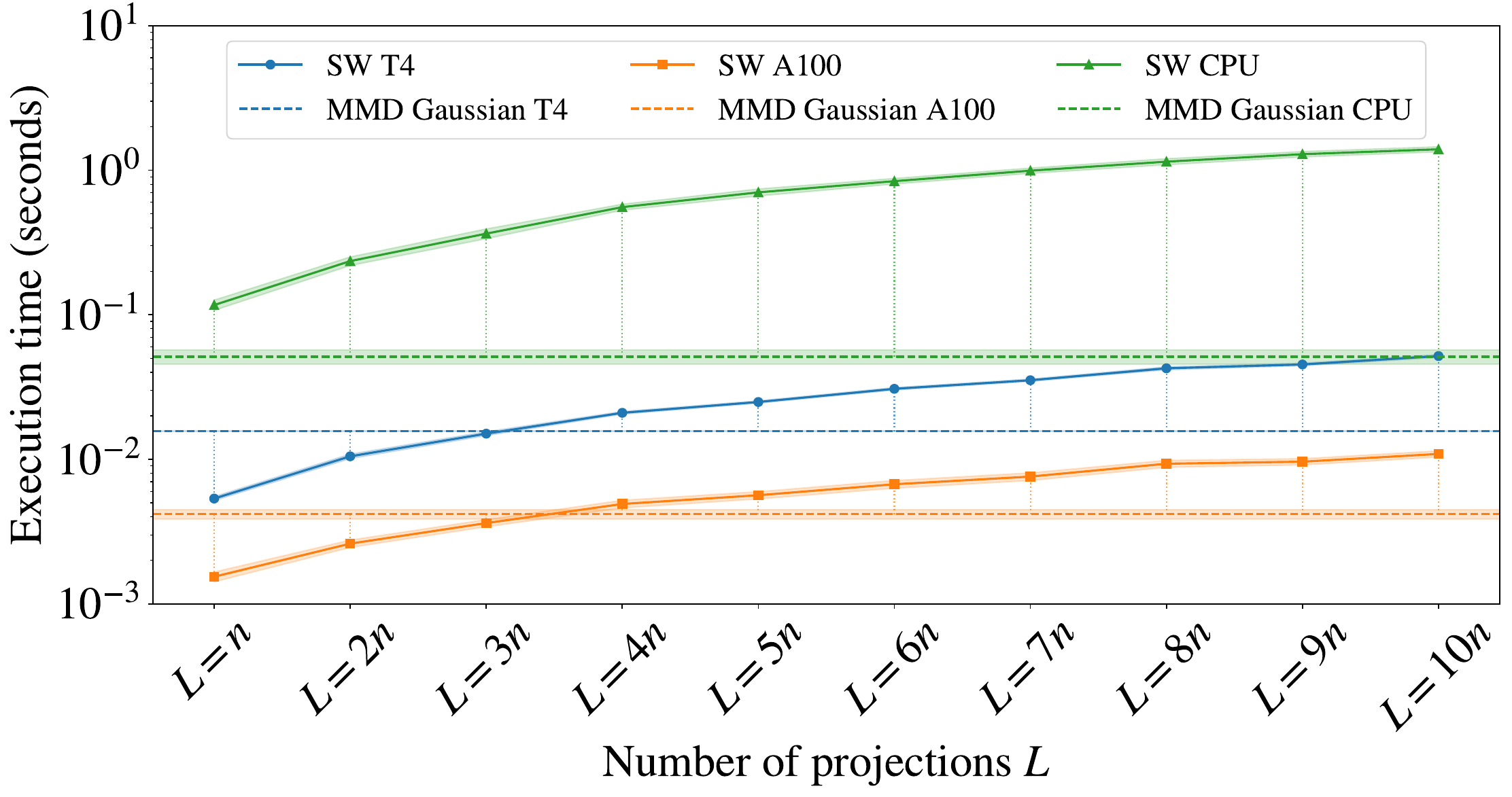}
  \caption{Computation time (log scale) of SW tests}
  \label{fig:improvetime}
\end{figure}


Finally, to illustrate Theorem~\ref{controltypeI}, we assessed the empirical Type I error of the SW test. In each of 2000 repetitions, two independent samples of size $50$ were drawn from $\mathcal{N}(0,I_{60})$. The results, reported in Table~\ref{tab:typeIerror} for three choices of $L$, show that the permutation-based SW test effectively controls the Type I error rate even at small sample sizes.

\begin{table}[H]
  \centering
  \caption{Type I error for the SW test}
  \label{tab:typeIerror}
  \begin{tabular}{cc} 
    \toprule
    \textbf{Test Statistics} & \textbf{Type I Error} \\ 
    \midrule
    SW $(L=0.5n)$ & $0.04982 \pm 0.00543$ \\
    SW $(L=n)$    & $0.04844 \pm 0.00501$ \\ 
    SW $(L=2n)$   & $0.04961 \pm 0.00461$ \\ 
    \bottomrule
  \end{tabular}
\end{table}

\begin{figure}[h]
  \centering
  \includegraphics[width=0.9\linewidth]{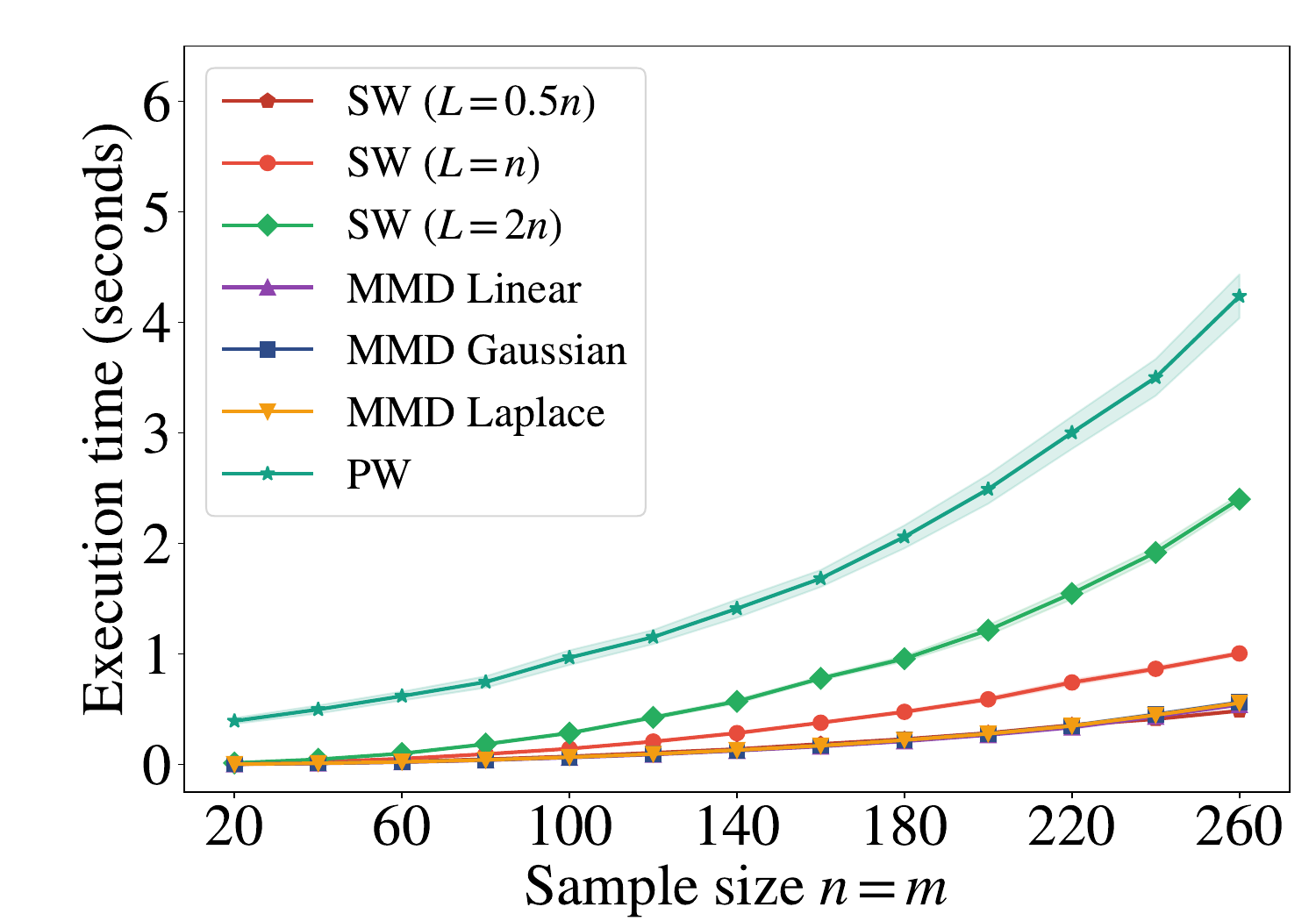}
  \caption{Computation time for ball vs. sphere \protect\footnotemark}
  \label{ballsphere_timing}
\end{figure}
\footnotetext{To leverage available CPU resources, we perform independent repetitions in parallel across all CPU cores and report the mean wall-clock time per repetition. Due to heterogeneous core performance, the average time can exceed that observed when using a single core. All methods are timed on the same inputs within each repetition, so their relative ranking is unaffected.}

In summary, our experiments demonstrate that the SW test offers a robust, parameter-free alternative to kernel-based methods, balancing statistical power and computational efficiency, especially when leveraging parallel hardware. 

%% file: conclusion.tex
We proposed a permutation-based two-sample test using the sliced Wasserstein distance, establishing finite-sample validity, non-asymptotic power bounds, and minimax optimality over multinomial and bounded-support alternatives. To our knowledge, this is the first Wasserstein-based test with finite-sample guarantees.
Our analysis quantified the trade-off between the number of projections and statistical power, and experiments showed that the test achieves consistently strong performance across benchmarks without parameter tuning, while remaining scalable on parallel hardware. Future directions include extending minimax optimality results to broader smoothness classes and exploring adaptive or non-uniform slicing strategies (e.g., generalized projections \citep{kolouri2019gsw}).

%% file: supplement.tex
\onecolumn
\aistatstitle{Supplementary Material\\
Minimax-Optimal Two-Sample Test with Sliced Wasserstein}
	
	
	

\section*{OVERVIEW}
This supplementary material contains additional discussions, detailed proofs, and an extra experiment that were omitted from the main text due to page constraints. The remaining sections are organized as follows.
\begin{itemize}
    \item In Section~\ref{Sec:Usefulconcentrations}, we recall several key concentration inequalities that serve as important tools for our proofs.
    \item In Section~\ref{Sec:OptimalTransportResult}, we derive intermediate technical lemmas for permuted sliced Wasserstein statistics using results from optimal transport theory. A core contribution of this section is the derivation of a novel concentration inequality and an expectation bound for the permuted sliced Wasserstein statistic. These results are fundamental for controlling the random permutation threshold in our power analysis.
    \item In Section~\ref{sec:SWcomplexity}, we present a result characterizing the sample complexity of the sliced Wasserstein distance.
    \item Our main theoretical contributions are presented in Sections~\ref{sec:upperboundproofs} and~\ref{sec:Lowerboundproof}. These sections provide the complete proofs of our main theorems concerning the test’s performance (Theorems~\ref{controltypeI} and~\ref{maintheorem}) and its minimax optimality (Propositions~\ref{minimaxmultinomial} and~\ref{lowerboundseperation}).
    \item In Section~\ref{sec:PermutationApproach}, we elaborate on why we chose a permutation-based framework for two-sample testing with the sliced Wasserstein distance, highlighting the intractability and practical challenges that prevent the use of the statistic's asymptotic null distribution.
    \item Section~\ref{sec:Technicalreview} offers a technical review of existing methods for analyzing permutation tests. We explain why common techniques developed for U-statistics and other test statistics are not directly applicable to the sliced Wasserstein distance, thereby underscoring the novelty of our analytical approach.
    \item Finally, Section~\ref{sec:addtionalexperiment} provides results for the Gaussian mean shift experiment mentioned in Section~\ref{sec:numerical_experiments}.
\end{itemize}
\paragraph{Additional Notation.} Throughout this supplementary material, we use an additional set of notation, described below.
\begin{itemize}
    \item For $x \in \mathbb{R}$, $\lceil x \rceil$ denotes the smallest integer greater than or equal to $x$, and $\lfloor x \rfloor$ denotes the largest integer less than or equal to $x$.
    \item We write $d_{\mathrm{TV}}(P,Q)$ for the total variation (TV) distance between $P$ and $Q$, and $D_{\mathrm{KL}}(P\|Q)$ for the Kullback--Leibler (KL) divergence. More details on those metrics can be found in, e.g., \citet[Section~2.4]{tsybakov2008introduction}.
    \item For any integer $K \geq 1$, we denote by $[K]$ the set of the first $K$ integers $\{1, \dots, K\}$.
    \item We use $\mathbb{P}_{\mu \times \nu}$ to denote probability with respect to $Y_1,\dots,Y_n \sim \mu$ and $Z_1,\dots,Z_m \sim \nu$. When additional randomness is present, we extend the notation accordingly, e.g., $\mathbb{P}_{\mu \times \nu \times r \times \sigma}$. 
    \item Finally, given two random variables $U$ and $V$, we write $U\stackrel{(d)}{=}V$ to denote that $U$ and $V$ have the same distribution.
\end{itemize}
\newpage
\section{SOME USEFUL CONCENTRATION INEQUALITIES}\label{Sec:Usefulconcentrations}
We begin this section by recalling two classical concentration inequalities: McDiarmid's inequality \citep{mcdiarmid1989method} and Hoeffding's inequality \citep{hoeffding1963probability}.
\begin{lemma}[McDiarmid's Inequality]\label{lem:MCDiarmid} Let $N\geq1$.
Let $X_1,\dots,X_N$ be independent random variables taking values in a set $\mathcal{X}$, 
and let $f:\mathcal{X}^N \to \mathbb{R}$.
Assume that $f$ satisfies the \emph{bounded difference property}, i.e., 
there exist constants $c_1,\dots,c_N \geq 0$ such that for all $i\in \{1,\dots,N\}$ and for all $x_1,\dots,x_N,x_i'\in \mathcal{X}$,
\begin{align*}
\big| f(x_1,\dots,x_i,\dots,x_N) - f(x_1,\dots,x_i',\dots,x_N)\big| \leq c_i.
\end{align*}
Then, for any $t > 0$,
\begin{align*}
\mathbb{P}\!\left(\,\big|f(X_1,\dots,X_N)-\mathbb{E}[f(X_1,\dots,X_N)]\big|\ge t\,\right) 
\le 2\exp\!\left(-\dfrac{2t^2}{\sum_{i=1}^{N}c_i^2}\right).
\end{align*}
\end{lemma}
\begin{lemma}[Hoeffding's Inequality]\label{lem:hoeffding}
Let $N \geq 1$. Let $X_1,\dots,X_N$ be independent random variables  such that $a_i\le X_i \le b_i$ for $i=1,\dots,N$. Then, for any $t>0$,
\begin{align*}
\mathbb{P}\left[\left|\frac{1}{N}\sum_{i=1}^{N}X_i-\mathbb{E}\left(\frac{1}{N}\sum_{i=1}^{N}X_i\right)\right|\ge t\right]\le 2\exp\left(-\frac{2N^2t^2}{\sum_{i=1}^N(b_i-a_i)^2}\right).
\end{align*}
\end{lemma}
We shall also rely on a permutation-based version of the classical McDiarmid's inequality to control the empirical quantile $\widehat{c}^B_{1-\alpha,N}$ introduced in Algorithm~\ref{algorithm}. A statement of this result can be found in \citet[Theorem~6]{tolstikhin2017concentration} and \citet[Lemma~2]{el2009transductive}. 
To prepare for the statement of that result, we first introduce the following definition.
\begin{definition}[$(n,m)$-symmetric function]\label{def:symmetricfunction}
Let $N \geq 1$ and let $n$ be a positive integer such that $n<N$. Set $m\coloneqq  N-n$. 
A function $f: S_N \to \mathbb{R}$, defined on the symmetric group over $\{1,\dots,N\}$, is called $(n,m)$-symmetric if it remains invariant under the change of order of the first $n$ coordinates and/or last $m$ coordinates of any permutation $\pi
\in S_{N}$.
\end{definition}
We are now ready to state a permutation-based version of McDiarmid's inequality.
\begin{lemma}[Permutation McDiarmid's inequality]\label{mcdiarmid} Let $n$ and $N$ be positive integers such that $n < N$.
Let $\pi
$ be a random permutation drawn uniformly from the symmetric group over $\{1, \dots, N\}$, and let $f(\pi
)$ be an $(n,N-n)$-symmetric function such that there exists a constant $b > 0$ satisfying
\begin{align*}
|f(\pi
) - f(\pi
^{i,j})| \le b,
\end{align*}
for all $\pi
$, $i \in \{1,\dots,n\}$, $j\in \{n+1,\dots,N\}$, where $\pi
^{i,j}$ is the permutation obtained by transposing the $i$-th and $j$-th entries of $\pi
$. Then, for any $\epsilon > 0$,
\begin{align*}
\mathbb{P}_{{\pi}
}\left( f({\pi}
) - \mathbb{E}_{\pi}[f({\pi}
)] \geq \epsilon \right)
\leq
\exp\left( -\frac{2\epsilon^2}{n b^2} \cdot \frac{N - \frac{1}{2}}{N - n} \cdot \left(1 - \frac{1}{2 \max\{n, N - n\}} \right) \right).
\end{align*}
\end{lemma}
\section{PROPERTIES OF THE PERMUTED SLICED WASSERSTEIN STATISTIC}\label{Sec:OptimalTransportResult}
Our main objective in this section is to obtain a deterministic upper bound the (random) permuted test statistics 
$\widehat{\operatorname{SW}}^{p,\pi}_p$ introduced in~\eqref{eq:empericalpermu}. To this end, we exploit the properties of the sliced Wasserstein distance under permutation via a permutation-based version of McDiarmid’s inequality (recalled in Lemma~\ref{mcdiarmid}). 
It yields a high-probability bound relating the permuted statistics to its expectation. We then control the expectation term by applying the optimal matching bound from \citet[Corollary~5]{bobkov2021simple}.

We recall that all results are established under the assumption that the samples $\mathcal{Y}_n = (y_1, \dots, y_n)$ and $\mathcal{Z}_m = (z_1, \dots, z_m)$ satisfy $\|y_i\| \le D$ and $\|z_j\| \le D$ for all $i = 1, \dots, n$ and $j = 1, \dots, m$. Moreover, $\Theta = (\theta_1, \dots, \theta_L)$ denotes the projection directions 
drawn from the uniform distribution $\sigma$ on the unit sphere $\mathbb{S}^{d-1}$.

We first establish a high-probability concentration inequality for the permuted empirical sliced Wasserstein distance $\widehat{\operatorname{SW}}^{p,\pi}_p$ introduced in \eqref{eq:empericalpermu}. It is the first step towards proving our main theorem in Section~\ref{sec:Proofofmaintheorem}. 
\begin{lemma}\label{Sensitivity}
Let $p\ge 1$. Let $\pi$ be a permutation drawn uniformly from the symmetric group  on $\{1,\dots,n+m\}$, and let $\pi^{i,j}$ denote the permutation obtained from $\pi$ by exchanging its $i$-th and $j$-th elements, as defined in Lemma~\ref{mcdiarmid}. Then, for any $\epsilon > 0$, 
\begin{align*}
\mathbb{P}_{\pi}\Bigg(
    \widehat{\operatorname{SW}}^{p,\pi}_p 
    - \mathbb{E}_{\pi}\!\left[\widehat{\operatorname{SW}}^{p,\pi}_p\right]
    \;\ge\; \epsilon 
    \,\Big|\, \mathcal{Y}_n, \mathcal{Z}_m, \Theta
\Bigg)
&\;\le\; 
\exp\!\left(
    -\frac{2\epsilon^2}{n\Big(\tfrac{(2D)^p}{n}+\tfrac{(2D)^p}{m}\Big)^2}
    \cdot \frac{n+m-\tfrac{1}{2}}{m}
    \left(1-\frac{1}{2\max\{m,n\}}\right)
\right).
\end{align*}
\end{lemma}
\begin{proof}
Given the samples $\mathcal{Y}_n, \mathcal{Z}_m$ and a set of projection directions $\Theta = (\theta_1,\dots,\theta_L)$ on the unit sphere $\mathbb{S}^{d-1}$, for brevity, we set 
\begin{align*}
f(\pi) \coloneqq  \widehat{\operatorname{SW}}^{p,\pi}_p, \qquad \pi \in S_{n+m},    
\end{align*}
where $\widehat{\operatorname{SW}}^{p,\pi}_p$ is the permuted empirical quantity introduced in \eqref{eq:empericalpermu}.

By construction, $f$ is $(n,m)$-symmetric. In order to apply Lemma~\ref{mcdiarmid} we need to get an upper bound on the sensitivity of $f$ to transpositions involving elements from opposite halves of the permutation.

Without loss of generality, suppose that $\pi=(1,2,3,\dots,n+m)$ and $\pi^{i,j}=(1,2,\dots, i-1, j, i+1, \dots,n,\dots, j-1, i, j+1, \dots n+m)$ for $1 \leq i  \leq n$ and $n+1 \leq j \leq n+m$.

Let $\ell \in \{1, \dots L\}$. Denote the projected samples on the direction $\theta_\ell$ as 
\begin{align*}
\tilde{\mathcal{X}} \coloneqq  (\tilde{x}_1,\dots,\tilde{x}_n,\tilde{x}_{n+1},\dots,\tilde{x}_{n+m})
= (\langle y_1,\theta_\ell\rangle,\dots,\langle y_n,\theta_\ell\rangle,
   \langle z_1,\theta_\ell\rangle,\dots,\langle z_m,\theta_\ell\rangle),    
\end{align*}
and accordingly,
\begin{align*}
\tilde{Y} \coloneqq  (\tilde{x}_1,\dots,\tilde{x}_n), 
\qquad 
\tilde{Z} \coloneqq (\tilde{x}_{n+1},\dots,\tilde{x}_{n+m}).  
\end{align*}

By definition of the $p$-Wasserstein distance  and the permutation $\pi$, we have
\begin{align*}
\operatorname{W}_p^p\!\Big( \Pi^{\theta_\ell}_{\#}\widehat{\mu}^\pi_n,\,
\Pi^{\theta_\ell}_{\#}\widehat{\nu}^{\pi}_m \Big) 
&= \min_{\substack{\gamma \geq 0, \\ \sum_k \gamma_{kl} = \tfrac{1}{m}, \\ \sum_l \gamma_{kl} = \tfrac{1}{n}}}
\left( \sum_{k=1}^n \sum_{l=1}^m 
\gamma_{kl}\, \big| \tilde{x}_{\pi(k)} - \tilde{x}_{\pi(n+l)} \big|^p \right)\\
&=\min_{\substack{\gamma \geq 0, \\ \sum_k \gamma_{kl} = \tfrac{1}{m}, \\ \sum_l \gamma_{kl} = \tfrac{1}{n}}}
\left( \sum_{k=1}^n \sum_{l=1}^m 
\gamma_{kl}\, \big| \tilde{y}_{k} - \tilde{z}_{l} \big|^p \right)\\
&=\sum_{k,l}\gamma^*_{kl}\big| \tilde{y}_k - \tilde{z}_l\big|^p,
\end{align*}
where $\gamma^*$ denotes the minimizer of the above Monge problem. 

We now consider the permutation $\pi^{i, j}$. Noting that that $\gamma^*$ is a valid coupling, we have 
\begin{align*}
\operatorname{W}_p^p\!\Big( \Pi^{\theta_\ell}_{\#}\widehat{\mu}^{\pi^{i,j}}_n,\,
\Pi^{\theta_\ell}_{\#}\widehat{\nu}^{\pi^{i,j}}_m \Big)
&= \min_{\substack{\gamma \geq 0, \\ \sum_k \gamma_{kl} = \tfrac{1}{m}, \\ \sum_l \gamma_{kl} = \tfrac{1}{n}}}
\left( \sum_{k=1}^n \sum_{l=1}^m 
\gamma_{kl}\, \big| \tilde{x}_{\pi^{i,j}(k)} - \tilde{x}_{\pi^{i,j}(n+l)} \big|^p \right)\\
&\le \sum_{k=1}^n \sum_{l=1}^m \gamma^*_{kl}\, \big| \tilde{x}_{\pi^{i,j}(k)} - \tilde{x}_{\pi^{i,j}(n+l)} \big|^p\\
&\le \sum_{k=1}^n \sum_{l=1}^m \gamma^*_{kl}\, \big|\tilde{y}_k - \tilde{z}_l\big|^p 
+ \sum_{l=1}^m \gamma^*_{ik}\, \big|\tilde{x}_{\pi^{i,j}(i)}- \tilde{x}_{\pi^{i,j}(n+l)}\big|^p  + \sum_{k=1}^n \gamma^*_{kj}\, \big|\tilde{x}_{\pi^{i,j}(k)} - \tilde{x}_{\pi^{i,j}(j)}\big|^p \\
&\le \operatorname{W}_p^p\!\Big( \Pi^{\theta_\ell}_{\#}\widehat{\mu}^\pi_n,\,
\Pi^{\theta_\ell}_{\#}\widehat{\nu}^\pi_m \Big)
+ \frac{(2D)^p}{n} + \frac{(2D)^p}{m},
\end{align*}
where the last inequality follows from the triangle inequality and the  fact that every projected sample satisfies 
\begin{align*}
    \big|\tilde{x}_h\big| = \big|\langle x_h,\theta_\ell\rangle\big| \le \|x_h\|\,\|\theta_\ell\| \le D,
\end{align*}
for $1\le h\le n+m$.
As a result, we obtain
\begin{align*}
f(\pi)-f(\pi^{i,j})
\le \frac{(2D)^p}{n}+\frac{(2D)^p}{m}.
\end{align*}
Following the same steps, we can prove that
\begin{align*}
f(\pi^{i,j})- f(\pi)
\le \frac{(2D)^p}{n}+\frac{(2D)^p}{m},
\end{align*}
which is sufficient to conclude that
\begin{align*}
\left|f(\pi)-f(\pi^{i,j})\right|
\le \frac{(2D)^p}{n}+\frac{(2D)^p}{m}.
\end{align*}
Applying Lemma~\ref{mcdiarmid} with $b=\frac{(2D)^p}{n}+\frac{(2D)^p}{m}$ and $f(\pi)=\operatorname{\widehat{SW}}^{p,\pi}_{p}$ completes the proof.
\end{proof}
Having successfully related the permuted test statistic $\widehat{\operatorname{SW}}^{p,\pi}_p$ (introduced in~\eqref{eq:empericalpermu}) to its expectation, we now turn to controlling this expectation term.
To proceed, we next present several auxiliary lemmas that will be used later in the proof.

Our argument begins with the following result, stated by \citet[ Corollary~5]{bobkov2021simple}.
\begin{lemma}[Corollary 5 in \cite{bobkov2021simple}]\label{lem:bobkov}
Let $n$ and $N$ be integers such that $1 \le n \le N$. Let $x_1,\dots,x_N \in [0,1]^d$. 
Denote by $\mathcal{G}_n$ the collection of all subsets $\tau \subset \{1,\dots,N\}$ of cardinality $|\tau| = n$, equipped with the uniform probability measure $\pi_n$. 
With every $\tau \in \mathcal{G}_n$, we associate the empirical measure
\begin{align*}
\mu_\tau \coloneqq  \frac{1}{n} \sum_{j \in \tau} \delta_{x_j}, \quad \tilde{\mu} \coloneqq  \frac{1}{N} \sum_{j=1}^N \delta_{x_j}.    
\end{align*}
Then the empirical measures $\mu_{\tau}$ satisfy
\begin{align*}
\mathbb{E}_{\pi_n} \left[ W_1(\mu_\tau, \tilde{\mu}) \right] \le 
\begin{cases}
\sqrt{\dfrac{2}{n}} & \text{if } d = 1, \\
8 \sqrt{\dfrac{1 + \log(2n)}{n}} & \text{if } d = 2, \\
\dfrac{13\sqrt{d}}{n^{1/d}} & \text{if } d \ge 3.
\end{cases}
\end{align*}
\end{lemma}
\begin{remark}
The result of Corollary~5 in \citet{bobkov2021simple} is stated for empirical measures supported on the unit cube $[0,1]^d$. We now extend this result to the case where the data points lie in a general cube $[-D,D]^d$ for some $D>0$.
To relate these two settings, consider the affine transformation
\begin{align*}
T \colon [0,1]^d \to [-D,D]^d, 
\qquad 
T(x) = 2D(x - \mathbf{1}_d),
\end{align*}
where $\mathbf{1}_d \coloneqq (1, 1, \dots, 1) \in \mathbb{R}^d$. Given data points $x_1, \dots, x_N \in [0,1]^d$, we define 
$$
x_i' = T(x_i), \quad i = 1, \dots, N,
$$
so that each $x_i'$ lies in the cube $[-D,D]^d$.

Let $\mu'_\tau$ and $\tilde{\mu}'$ denote the empirical measures defined analogously to $\mu_\tau$ and $\tilde{\mu}$ in Lemma~\ref{lem:bobkov} but based on points $(x_i')_{i=1}^N$, that is
\begin{align*}
\mu_\tau' = \frac{1}{n}\sum_{j\in\tau}\delta_{x_j'}, 
\qquad 
\tilde{\mu}' = \frac{1}{N}\sum_{j=1}^N \delta_{x_j'}.
\end{align*}
By construction, we have $T_\# \mu_\tau = \mu_\tau'$ and $T_\# \tilde{\mu} = \tilde{\mu}'$. Moreover, for any $x,y \in [0,1]^d$, we have
\begin{align*}
\|T(x)-T(y)\|_2=2D\|x-y\|_2.
\end{align*}
Consequently, we have the scaling relation
\begin{align*}
W_1(\mu'_\tau, \tilde{\mu}') = 2D \, W_1(\mu_\tau, \tilde{\mu}).   
\end{align*}
Applying Corollary~5 in \citet{bobkov2021simple} to $\mu_\tau$ and $\tilde{\mu}$ then yields
\begin{align*}
\mathbb{E}_{\pi_n}\!\left[ W_1(\mu'_\tau, \tilde{\mu}') \right]
=\mathbb{E}_{\pi_n} \left[(2D)W_1(\mu_\tau, \tilde{\mu}) \right] \le 
\begin{cases}
2D \sqrt{\dfrac{2}{n}} & \text{if } d = 1, \\
16D \sqrt{\dfrac{1 + \log(2n)}{n}} & \text{if } d = 2, \\
\dfrac{26D\sqrt{d}}{n^{1/d}} & \text{if } d \ge 3.
\end{cases}
\end{align*}
\end{remark}
\begin{lemma}[Relation between $1$-Wasserstein and $p$-Wasserstein distances]\label{connectionwasserstein}
Let $P$, $Q$, and $M$ be three probability measures supported on the compact set $[-D,D]^d \subset \mathbb{R}^d$. Let $p \ge 1$, then the following inequality holds:
\begin{align*}
\operatorname{W}_p^p(P,Q) \le (4D\sqrt{d})^{p-1} \left( \operatorname{W}_1(P,Q) + \operatorname{W}_1(Q, M) \right).
\end{align*}
\end{lemma}
\begin{proof}
By applying the triangle inequality for the $p$-Wasserstein distance, we get
\begin{align*}\operatorname{W}_p(P,Q) \le \operatorname{W}_p(P, Q) + \operatorname{W}_p(Q, M).\end{align*}
Using the inequality $(a + b)^p \le 2^{p-1}(a^p + b^p)$, we obtain
\begin{align*}
\operatorname{W}_p^p(P,Q) \le 2^{p-1} \left(\operatorname{W}_p^p(P,M) + \operatorname{W}_p^p(Q,M) \right).
\end{align*}
For $x, y \in [-D,D]^d$, we have
\begin{align*}\|x - y\|_2 \le 2D\sqrt{d}.\end{align*}
Moreover, for $p\ge 1$, we have
\begin{align*}\|x - y\|_2^p \le (2D\sqrt{d})^{p-1}\|x-y\|_2.\end{align*}
Hence, since all three measures are supported on $[-D,D]^d$, by the definition of Wasserstein distances, we obtain
\begin{align*}\operatorname{W}_p^p(P,M) \le (2D\sqrt{d})^{p-1}\operatorname{W}_1(P,M), \quad \operatorname{W}_p^p(Q,M) \le (2D\sqrt{d})^{p-1}\operatorname{W}_1(Q,M).\end{align*}
Combining all the inequalities gives the stated result.
\end{proof}
\begin{remark}
We recall that our standing assumption is that all samples $\mathcal{Y}_n = (y_1, \dots, y_n)$ and $\mathcal{Z}_m = (z_1, \dots, z_m)$ satisfy $\|y_i\| \le D$ and $\|z_j\| \le D$ for all $i = 1, \dots, n$ and $j = 1, \dots, m$. 
Meanwhile, the results in Lemma~\ref{lem:bobkov} and Lemma~\ref{connectionwasserstein} are stated for the case where all data points lie in the cube $[-D, D]^d$. 
Since the Euclidean ball $\{x \in \mathbb{R}^d : \|x\| \le D\}$ is contained in this cube, those results remain valid in our setting.
\end{remark}
Using the results of Lemma~\ref{lem:bobkov} and Lemma~\ref{connectionwasserstein}, along with their accompanying remarks, the following proposition establishes an upper bound for the expectation term conditional on the given samples and the projection directions.
\begin{proposition}\label{boundexpectation} Let $n$ and $m$ be positive integers such that $n \leq m$, $p \ge 1$ and set $N\coloneqq  n+m$. Given a collection of points $x_1,\dots,x_{N}$ with $\|x_i\|\leq D$ for all $1 \leq i \leq N$, 
and a set of projection directions $\Theta = (\theta_1,\dots,\theta_{L})$ on the unit sphere $\mathbb{S}^{d-1}$. Let $\pi$ be a random permutation drawn uniformly from the symmetric group over $\{1,\dots,N\}$. Then, with $p\ge 1$, we have
\begin{align*}
\mathbb{E}_{\pi}\left[\widehat{\operatorname{SW}}^{p,\pi}_p\right]=\frac{1}{L}\sum_{\ell =1}^{L}\mathbb{E}_{\pi}\left[\operatorname{W}_p^p\!\left(\Pi^{\theta_\ell}_{\#}\widehat{\mu}^{\pi}_n,\,
\Pi^{\theta_\ell}_{\#}\widehat{\nu}^{\pi}_m \right)\right]\le (4D)^{\,p} \frac{\sqrt{2}}{\sqrt{n}}.
\end{align*}
Here, the expectation $\mathbb{E}_\pi$ is taken with respect to the random permutation $\pi$, conditional on the samples and the projection directions.
\end{proposition}
\begin{proof}
Fix a projection direction $\theta_{\ell}$, $1 \leq \ell \leq L$. We want to upper bound
$\mathbb{E}_{\pi}\left[\operatorname{W}_p^p\!\left( \Pi^{\theta_\ell}_{\#}\widehat{\mu}^{\pi}_n,\,
\Pi^{\theta_\ell}_{\#}\widehat{\nu}^{\pi}_m \right)\right].$

We introduce the empirical measure
\begin{align*}
\tilde{\mu}\coloneqq  \frac{1}{N}\sum_{j=1}^{N}\delta_{\langle x_j, \theta_\ell\rangle}.
\end{align*}
Note that each projected sample satisfies $\lvert \langle x_j, \theta_\ell\rangle \rvert \leq D$ for all $1 \le j \le N$ by the Cauchy--Schwarz inequality.

Adapting the notations from Lemma~\ref{lem:bobkov}, we first define 
$\mathcal{G}_n$ and $\mathcal{G}_m$ as the collections of all subsets of 
$\{1, \dots, n+m\}$ of cardinalities $n$ and $m$, respectively. 
We denote by $\pi_n$ and $\pi_m$ the uniform distributions on 
$\mathcal{G}_n$ and $\mathcal{G}_m$. To every $\tau^n \in \mathcal{G}_n$ and 
$\tau^m \in \mathcal{G}_m$, we associate two empirical measures
\begin{align*}
\mu_{\tau^n}\coloneqq \frac{1}{n}\sum_{j \in \tau^n} 
   \delta_{\langle \theta_{\ell}, x_j \rangle}, 
\qquad 
\mu_{\tau^m}\coloneqq  \frac{1}{m}\sum_{j \in \tau^m} 
   \delta_{\langle \theta_{\ell}, x_j \rangle}.
\end{align*}
It follows that
\begin{align*}
\mathbb{E}_{\pi}\left[\operatorname{W}_p^p\!\left( \Pi^{\theta_\ell}_{\#}\widehat{\mu}^{\pi}_n,\,
\Pi^{\theta_\ell}_{\#}\widehat{\nu}^{\pi}_m \right)\right]&=\frac{1}{(n+m)!}\sum_{\pi \in S_{n+m}}\operatorname{W}^p_p\left(\frac{1}{n} \sum_{i=1}^{n} \delta_{\langle x_{\pi(i)},\theta_\ell\rangle}, \frac{1}{m} \sum_{i=1}^{m} \delta_{\langle x _{\pi(n+i)},\theta_\ell \rangle}\right)\\
&=\frac{m!.n!}{(m+n)!} \sum_{K\subset\{x_{1},\dots,x_{n+m}\},|K|=n} \operatorname{W}^p_p\left(\frac{1}{n}\sum_{x \in K}\delta_{\langle x,\theta_\ell\rangle},\frac{1}{m}\sum_{x \in K^{c}} \delta_{\langle x,\theta_\ell\rangle}\right)\\
&\overset{(i)}{\le} \frac{(4D)^{p-1}}{\binom{n+m}{n}} \sum_{K \subset \{x_1,\dots,x_{n+m}\},|K|=n}\left[\operatorname{W}_1\left(\frac{1}{n}\sum_{x \in K}\delta_{\langle x,\theta_\ell\rangle},\tilde{\mu}\right)+\operatorname{W}_1\left(\frac{1}{m}\sum_{x \in K^c}\delta_{\langle x,\theta_\ell\rangle},\tilde{\mu}\right)\right]\\
& = (4D)^{p-1} \left[\mathbb{E}_{\pi_m}(\operatorname{W}_1(\mu_{\tau^m},\tilde{\mu}))+\mathbb{E}_{\pi_n}(\operatorname{W}_1(\mu_{\tau^n},\tilde{\mu}))\right]\\
&\overset{(ii)}{\le} 2D \cdot (4D)^{\,p-1} \left(\frac{\sqrt{2}}{\sqrt{n}}+\frac{\sqrt{2}}{\sqrt{m}}\right) \\
&\overset{(iii)}{\le} 4D \cdot (4D)^{\,p-1} \frac{\sqrt{2}}{\sqrt{n}},
\end{align*}
where $(i)$ follows from Lemma~\ref{connectionwasserstein}, and $(ii)$ and $(iii)$ follow from Lemma~\ref{lem:bobkov} and from the assumption that $n\le m$, respectively.

Finally, by linearity of the expectation, 
\begin{align*}
\mathbb{E}_{\pi}\left[\widehat{\operatorname{SW}}^{p,\pi}_p\right]=\frac{1}{L}\sum_{\ell =1}^{L}\mathbb{E}_{\pi}\left[\operatorname{W}_p^p\!\left( (\Pi^{\theta_\ell})_{\#}\widehat{\mu}^{\pi}_n,\,
(\Pi^{\theta_\ell})_{\#}\widehat{\nu}^{\pi}_m \right)\right]\le (4D)^{\,p} \frac{\sqrt{2}}{\sqrt{n}},
\end{align*}
and the proof is concluded.
\end{proof}
\section{SLICED WASSERSTEIN SAMPLE COMPLEXITY}\label{sec:SWcomplexity}

Let $\mu$ and $\nu$ be probability distributions on $\mathbb{R}^d$ whose supports lie in a common ball centered at the origin with radius $D$. Their corresponding empirical measures are denoted by $\widehat{\mu}_n$ and $\widehat{\nu}_m$. The quantities $\operatorname{SW}^p_p(\mu,\nu)$ and $\widehat{\operatorname{SW}}^p_p(\widehat{\mu}_n,\widehat{\nu}_m)$ are defined in Eqs.~\eqref{definitionSlicedWasserstein} and~\eqref{teststatistic}, respectively.

To the best of our knowledge, recent works have established bounds on quantities such as $\mathbb{E}\left[\operatorname{SW}_p(\widehat{\mu}_n,\mu)\right]$ (see, e.g., \citet[Theorem~1]{nietert2022statistical}), $\mathbb{E}\left[\operatorname{SW}^p_p(\widehat{\mu}_n,\widehat{\nu}_n)-\operatorname{SW}^p_p(\mu,\nu)\right]$ (see, e.g., \citet[Theorem~2]{ohana2023shedding}), and $\mathbb{E}\left[\big|\widehat{\operatorname{SW}}^p_p(\widehat{\mu}_n,\widehat{\nu}_m)-\operatorname{SW}^p_p(\mu,\nu)\big|\right]$ (see, e.g., \citet[Proposition~5]{nietert2022statistical}), as well as high-probability bounds for $\big|\operatorname{SW}^p_p(\mu,\nu)-\widehat{\operatorname{SW}}^p_p(\mu,\nu)\big|$ (see, e.g., \citet[Proposition~4]{xu2022central}). These results, however, do not exactly yield the type of bound required for the subsequent step—specifically, in the proof of Lemma~\ref{lem:type2}.

Under the assumption that both distributions $\mu$ and $\nu$ are supported on a common bounded ball, we derive the following high-probability bound for $\big|\widehat{\operatorname{SW}}^p_p(\widehat{\mu}_n,\widehat{\nu}_m)
- \operatorname{SW}^p_p(\mu,\nu)\big|$ by applying standard concentration inequalities—namely McDiarmid’s inequality and Hoeffding’s inequality (recalled in Lemma~\ref{lem:MCDiarmid} and Lemma~\ref{lem:hoeffding}, respectively). The resulting statement is presented in the following lemma.
\begin{lemma}\label{lem:concentration}
Let $\beta \in (0,1)$ and $p \ge 1$. 
Suppose that $\mu$ and $\nu$ are probability distributions supported 
on the centered ball of radius $D > 0$ in $\mathbb{R}^d$. 
We consider two independent samples drawn from $\mu$ and $\nu$:
\begin{equation*}
\mathcal{Y}_n\coloneqq(Y_1, \dots, Y_n) \overset{\text{i.i.d.}}{\sim} \mu, 
\qquad 
\mathcal{Z}_m\coloneqq (Z_1, \dots, Z_m) \overset{\text{i.i.d.}}{\sim} \nu,
\end{equation*}
with an assumption that $n \le m$.  
Let $\Theta\coloneqq (\theta_1, \dots, \theta_L)$ denote $L$ i.i.d.\ projection directions drawn from 
the uniform distribution $\sigma$ on the unit sphere $\mathbb{S}^{d-1}$. 
Then, with probability at least $1 - \tfrac{\beta}{2}$, the following inequality holds: 
\begin{align*}
\big|\widehat{\operatorname{SW}}^p_p(\widehat{\mu}_n,\widehat{\nu}_m)
- \operatorname{SW}^p_p(\mu,\nu)\big| \;<\; (2D)^p \left(
\sqrt{\frac{\log(8/\beta)}{2L}}
     + \sqrt{\frac{\log(8/\beta)}{n}}
   \right).
\end{align*}
\end{lemma}
\begin{proof}
We decompose the total error into two contributions:
\begin{align*}
\big|\widehat{\operatorname{SW}}^p_p(\widehat{\mu}_n,\widehat{\nu}_m)
- \operatorname{SW}^p_p(\mu,\nu)\big|
\;\le\; A + B,
\end{align*}
where
\begin{align*}
A = \big|\widehat{\operatorname{SW}}^p_p(\widehat{\mu}_n,\widehat{\nu}_m) 
- \operatorname{SW}^p_p(\widehat{\mu}_n,\widehat{\nu}_m)\big|,
\quad 
B = \big|\operatorname{SW}^p_p(\widehat{\mu}_n,\widehat{\nu}_m) 
- \operatorname{SW}^p_p(\mu,\nu)\big|.
\end{align*}
\paragraph{Control of $A$.}

Given the datasets $\mathcal{Y}_n, \mathcal{Z}_m$, the Monte Carlo estimator $\widehat{\operatorname{SW}}^p_p(\widehat{\mu}_n,\widehat{\nu}_m)$ is an average of $L$ i.i.d. bounded random variables 
\begin{align*}
R_\ell = \operatorname{W}_p^p\!\big(\Pi^{\theta_\ell}_\#\widehat{\mu}_n, \Pi^{\theta_\ell}_\#\widehat{\nu}_m\big), 
\qquad \ell = 1, \dots, L.
\end{align*}
Since $\|\theta_\ell\|_2=1$ and the supports of $\mu$ and $\nu$ are contained in the centered ball with radius $D$, the push-forward measures $(\Pi^{\theta_\ell})_{\#}\widehat{\mu}_n$ and $(\Pi^{\theta_\ell})_{\#}\widehat{\nu}_m$ 
are supported on the interval $[-D,D]$. 
Consequently,
\begin{align*}
0 \;\le\; R_\ell \;\le\; (2D)^p, 
\qquad \forall\, 1 \le \ell \le L.
\end{align*}
Then, Hoeffding's inequality (recalled in Lemma~\ref{lem:hoeffding}) yields that, for any $t_A>0$,
\begin{align*}
\mathbb{P}_\sigma\!\left(A \ge t_A \;\middle|\; \mathcal{Y}_n,\mathcal{Z}_m\right)
\;\le\; 2\exp\!\left(-\dfrac{2L t_A^2}{(2D)^{2p}}\right).
\end{align*}
Since the right-hand side does not depend on the samples, the same bound holds unconditionally, namely
\begin{align*}
\mathbb{P}_{\mu,\nu,\sigma}(A \ge t_A) \;\le\; 2\exp\!\left(-\dfrac{2L t_A^2}{(2D)^{2p}}\right).
\end{align*}
\paragraph{Control of $B$.}
Fix the projection directions $\theta_1, \dots, \theta_L$.

The function  $(y_1,\dots,y_n,z_1,\dots,z_m) \mapsto \operatorname{SW}^p_p(\widehat{\mu}_n,\widehat{\nu}_m)$ satisfies the bounded difference property, whose definition is recalled in Lemma~\ref{lem:MCDiarmid}.
Indeed, changing a single observation $y_i$ modifies $\widehat{\mu}_n$ by at most $1/n$ of its mass. 
Since all points are supported in a set of diameter $2D$, the value of $\operatorname{SW}^p_p(\widehat{\mu}_n,\widehat{\nu}_m)$ changes by at most $(2D)^p/n$. 
Similarly, replacing one sample $z_j$ alters $\widehat{\nu}_m$ by $1/m$ of its mass, 
and hence $\operatorname{SW}^p_p(\widehat{\mu}_n,\widehat{\nu}_m)$ changes by at most $(2D)^p/m$.
McDiarmid’s inequality (recalled in Lemma~\ref{lem:MCDiarmid}) guarantees that, for any $t_B>0$,
\begin{align*}
\mathbb{P}_{\mu,\nu}\left(B\ge t_B\mid \theta_1,\dots,\theta_L\right)\le 2\exp\!\left(-\dfrac{2t_B^2}{(2D)^{2p}(1/n+1/m)}\right).
\end{align*}
Since the right-hand side is independent of the projection directions, the bound also holds without conditioning, that is,
\begin{align*}
\mathbb{P}_{\mu,\nu,\sigma}\left(B\ge t_B\right)\;\le\; 2\exp\!\left(-\dfrac{2t_B^2}{(2D)^{2p}(1/n+1/m)}\right).
\end{align*}
Moreover, since $n\le m$, then 
\begin{align*}\mathbb{P}_{\mu,\nu,\sigma}\!\left(B \ge t_B\right)
\;\le\; 2 \exp\!\left(
-\,\frac{nt_B^2}{(2D)^{2p}\,}
\right).\end{align*}

A simple union bound gives
\begin{align*}
\mathbb{P}_{\mu,\nu,\sigma}\!\left(A+B \;\ge\; t_A + t_B\right)
\;\le\; \mathbb{P}_{\mu,\nu,\sigma}(A \ge t_A)
+ \mathbb{P}_{\mu,\nu,\sigma}(B \ge t_B).
\end{align*}
Hence,
\begin{align*}
\mathbb{P}_{\mu,\nu,\sigma}\!\left(
\big|\widehat{\operatorname{SW}}_p^p(\widehat{\mu}_n,\widehat{\nu}_m) - \operatorname{SW}_p^p(\mu,\nu)\big|
\;\ge\; t_A + t_B
\right)
\;\le\;
2\exp\!\left(-\frac{2L t_A^2}{(2D)^{2p}}\right)
+2\exp\!\left(-\frac{n t_B^2}{(2D)^{2p}}\right).
\end{align*}
We obtain the stated result by setting
\begin{align*} 
	t_A \;=\; (2D)^p \sqrt{\frac{\log(8/\beta)}{2L}} \quad \text{ and } \quad 
	t_B \;=\; (2D)^p \sqrt{\frac{\log(8/\beta)}{n}}.
\end{align*}
\end{proof}
\section{LEVEL AND POWER GUARANTEES FOR ALGORITHM~\ref{algorithm}}\label{sec:upperboundproofs}
 \allowdisplaybreaks
\subsection{Proof of Theorem~\ref{controltypeI}}
Following the proof of \citet[Proposition~1]{schrab2023mmd}, we obtain the following chain of implications
\begin{align*}
\Delta\left(\mathcal{Y}_n,\mathcal{Z}_m,\mathbb{Z}_{B},\Theta\right) = 1 
&\Rightarrow \widehat{\operatorname{SW}}^p_p> \widehat{c}^{B}_{1-\alpha,N} \\
&\Rightarrow \widehat{\operatorname{SW}}^p_p > \widehat{\operatorname{SW}}^{p, r_{\bullet \lceil (B+1)(1-\alpha) \rceil}}_p \\
&\Rightarrow\sum_{b=1}^{B+1} \mathbf{1}\left( \widehat{\operatorname{SW}}^{p,\pi_b}_p < \widehat{SW}^{p,\pi_{B+1}}_p \right) \ge \lceil (B+1)(1-\alpha) \rceil \\
&\Rightarrow B+1 - \sum_{b=1}^{B+1} \mathbf{1}\left( \widehat{\operatorname{SW}}^{p,\pi_b}_p < \widehat{\operatorname{SW}}^{p,\pi_{B+1}}_p \right) \le B+1 - \lceil (B+1)(1-\alpha) \rceil \\
&\Rightarrow
\sum_{b=1}^{B+1} \mathbf{1}\left( \widehat{\operatorname{SW}}^{p,\pi_b}_p \ge \widehat{\operatorname{SW}}^{p,\pi_{B+1}}_p \right) \le \lfloor \alpha(B+1) \rfloor
\\
&\Rightarrow \sum_{b=1}^{B+1} \mathbf{1}\left( \widehat{\operatorname{SW}}^{p,\pi_b}_p \ge \widehat{SW}^{p,\pi_{B+1}}_p \right) \le \alpha(B+1) \\
&\Rightarrow \frac{1}{B+1} \left[1 + \sum_{b=1}^{B} \mathbf{1}\left( \widehat{\operatorname{SW}}^{p,\pi_b}_p \ge \widehat{\operatorname{SW}}^{p,\pi_{B+1}}_p \right) \right] \le \alpha,
\end{align*}
where the fifth deduction follows from the fact that
$B+1-\lceil (1-\alpha)(B+1) \rceil=\lfloor \alpha(B+1) \rfloor
$. Besides, the notations $\widehat{\operatorname{SW}}_p^p$
and $\widehat{\operatorname{SW}}_p^{p,\pi}$
refer respectively to the statistics defined in
\eqref{teststatistic} and~\eqref{eq:empericalpermu}.

Moreover, Lemma~1 in \cite{romano2005exact} guarantees that:
\begin{align*}
\mathbb{P}_{\mu \times \mu \times r\times \sigma} \left\{ \frac{1}{B+1} \left[ 1 + \sum_{b=1}^{B} \mathbf{1}\left( \widehat{\operatorname{SW}}^{p,\pi_b}_p \ge \widehat{\operatorname{SW}}^{p,\pi_{B+1}}_p \right) \right] \le \alpha \right\} \le \alpha.
\end{align*}
As a consequence, we have
\begin{align*}\mathbb{P}_{\mu \times \mu \times r\times \sigma}\left(\Delta\left(\mathcal{Y}_n,\mathcal{Z}_m,\mathbb{Z}_{B},\Theta\right)=1\right)\le \mathbb{P}_{\mu \times \mu \times r\times \sigma} \left\{ \frac{1}{B+1} \left[ 1 + \sum_{b=1}^{B} \mathbf{1}\left( \widehat{\operatorname{SW}}^{p,\pi_b}_p \ge \widehat{\operatorname{SW}}^{p,\pi_{B+1}}_p \right) \right] \le \alpha \right\} \le \alpha.
\end{align*}
Hence, the test controls the Type~I error at level $\alpha$.
\subsection{Proof of Theorem~\ref{maintheorem}}\label{sec:Proofofmaintheorem}

Let $r$ denote the uniform distribution over the symmetric group $S_{n+m}$, i.e., the set of all permutations of $\{1,\dots,n+m\}$. Let $\pi$ be a permutation drawn from $r$. Moreover, we defined $c_{1-\alpha,N}$ as the $(1-\alpha)$-quantile (with respect to the randomness of $\pi$) of the test statistics $\widehat{\operatorname{SW}}^{p,\pi}_p$ (introduced in Eq.~\eqref{eq:empericalpermu}), that is,
\begin{align}\label{recalquantile1}
c_{1-\alpha,N}=\inf\left\{t:\mathbb{P}_{\pi}\left(\widehat{\operatorname{SW}}^{p,\pi}_p\ge t\,\Big|\, \mathcal{Y}_n, \mathcal{Z}_m, \Theta\right)\le\alpha\right\}.
\end{align}

As explained in Remark~\ref{remark:MonteCarlopermutation}, in practice we draw $B$ independent permutations $(\pi_{b})_{1\le b \le B}$ from the uniform distribution $r$ on $S_{n+m}$ and set $\pi_{B+1} := \text{id}$ to denote the identity permutation. We then compute the corresponding statistics
$\widehat{\operatorname{SW}}^{p,\pi_b}_p$ for $b=1,\dots,B+1$ and estimate the $(1-\alpha)$ empiraical quantile of the permutation distribution by
\begin{align*}
\widehat{c}^B_{1-\alpha,N} := \inf \left\{ t : \frac{1}{B+1} \sum_{i=1}^{B+1} \mathbf{1}\{\widehat{\operatorname{SW}}^{p,\pi_i}_p \le t\} \ge 1-\alpha \right\}.
\end{align*}
In accordance with the convention introduced earlier, we use the notation $\mathbb{P}_{r}$ to denote the probability with respect to the $B$ random permutations $\pi_1, \dots, \pi_B \stackrel{\text{i.i.d.}}{\sim} r$.

With these preliminaries in place, the proof of Theorem~\ref{maintheorem} proceeds as follows. Taking inspiration from \citet[Lemma~4]{schrab2023mmd}, we first provide a condition on the separation between the distributions $\mu$ and $\nu$ that guarantees a desired level of test power (see Lemma~\ref{lem:type2}). This condition relates the separation to the (random) empirical quantile $\widehat{c}^B_{1-\alpha,N}$. Next, we control $\widehat{c}^B_{1-\alpha,N}$ by linking it to the (deterministic) quantile $c_{1-\alpha,N}$, conditional on the samples and the projection directions (see Lemma~\ref{lem:connection}). Then, the quantile $c_{1-\alpha,N}$ is bounded using the auxiliary results prepared in Section~\ref{Sec:OptimalTransportResult} (see Proposition~\ref{quantilebound}). Finally, we derive the bound for $\widehat{c}^B_{1-\alpha,N}$ stated in Proposition~\ref{boundempiricalquantile}. We now develop this program in detail.

\subsubsection*{Main ingredients for the proof}
We begin with the following lemma, which is a straightforward adaptation of \citet[Lemma~4]{schrab2023mmd}. It provides a sufficient condition on the separation between distributions $\mu$ and $\nu$ to ensure a desired level of test power.
\begin{lemma}\label{lem:type2}
Let $\beta \in (0,1)$ and $p\ge 1$. Consider the setting of Lemma~\ref{lem:concentration}. Let \begin{align*}
\gamma(n,p,\beta,D,L)
\coloneqq  (2D)^p\left(
\sqrt{\frac{\log(8/\beta)}{2L}}
+ \sqrt{\frac{\log(8/\beta)}{n}}
\right).    
\end{align*}
The test defined in Algorithm~\ref{algorithm} achieves power at least $1-\beta$ provided that
\begin{equation}\label{eq:type2-cond}
\mathbb{P}_{\mu\times \nu\times r \times \sigma}\left(
\operatorname{SW}^p_p(\mu,\nu)\ge
\gamma(n,p,\beta,D,L) + \widehat{c}^{B}_{1-\alpha,N}
\right) > 1-\tfrac{\beta}{2}.
\end{equation}
\end{lemma}
\begin{proof}
Lemma~\ref{lem:concentration} guarantees
\begin{equation}\label{eq6}
\mathbb{P}_{\mu\times\nu\times r \times \sigma}\!\left(
\big|\widehat{\operatorname{SW}}^p_p(\widehat{\mu}_n,\widehat{\nu}_m)
   - \operatorname{SW}^p_p(\mu,\nu)\big|
   \;\ge\; \gamma(n,p,\beta,D,L)
\right) 
\;\le\; \frac{\beta}{2},
\end{equation}
where
\begin{align*}
\gamma(n,p,\beta,D,L)
\coloneqq  (2D)^p\left(
\sqrt{\frac{\log(8/\beta)}{2L}}
+ \sqrt{\frac{\log(8/\beta)}{n}}
\right).    
\end{align*}

Define the events 
\begin{align*}
\mathcal{A}\coloneqq\{\widehat{\operatorname{SW}}^p_p(\widehat{\mu}_n,\widehat{\nu}_m) \le \widehat{c}^B_{1-\alpha,N}\} \quad \text{ and }
\mathcal{B}\coloneqq\{\operatorname{SW}^p_p(\mu,\nu)\ge \gamma(n,p,\beta,D,L)+\widehat{c}^B_{1-\alpha,N}\}.
\end{align*}
Let us show that $\mathbb{P}_{\mu \times \nu \times r\times \sigma}\left(\mathcal{A}\right)\le \beta$. By definition of the events and Eq.~\eqref{eq6}, we have  
\begin{align*}
\mathbb{P}_{\mu \times \nu \times r \times \sigma}(\mathcal{A} \cap \mathcal{B}) &= \mathbb{P}_{\mu \times \nu \times r \times \sigma}\left(\widehat{\operatorname{SW}}^p_p(\widehat{\mu}_n,\widehat{\nu}_m)\le \widehat{c}^B_{1-\alpha,N},\widehat{c}^B_{1-\alpha,N}\le \operatorname{SW}^p_p(\mu,\nu)-\gamma(n,p,\beta,D,L)\right)\\
&\le \mathbb{P}_{\mu \times \nu \times r\times \sigma}\left(\widehat{\operatorname{SW}}^p_p(\widehat{\mu}_n,\widehat{\nu}_m)-\operatorname{SW}^p_p(\mu,\nu) \le -\gamma(n,p,\beta,D,L)\right)\\
& \le \mathbb{P}_{\mu \times \nu \times r\times \sigma}\left(\left|\widehat{\operatorname{SW}}^p_p(\widehat{\mu}_n,\widehat{\nu}_m)-\operatorname{SW}^p_p(\mu,\nu)\right| \ge  \gamma(n,p,\beta,D,L)\right)\\
& \le \frac{\beta}{2}.
\end{align*}
To conclude, note that whenever $\mathbb{P}_{\mu\times \nu \times r \times \sigma}(\mathcal{B})> 1-\frac{\beta}{2}$, the theorem of total probability implies 
\begin{align*}
\mathbb{P}_{\mu \times \nu \times r \times \sigma}(\mathcal{A})
&= \mathbb{P}_{\mu \times \nu \times r \times \sigma}(\mathcal{A} \cap \mathcal{B})
  + \mathbb{P}_{\mu \times \nu \times r \times \sigma}(\mathcal{A} \cap \mathcal{B}^c) \\
&= \mathbb{P}_{\mu \times \nu \times r \times \sigma}\!\left(\mathcal{A} \cap \mathcal{B}\right)
  + \mathbb{P}_{p \times q \times r \times \sigma}\!\left(\mathcal{A}\mid \mathcal{B}^c\right)
    \mathbb{P}_{\mu \times \nu \times r \times \sigma}\!\left(\mathcal{B}^c\right) \\
&\le \frac{\beta}{2} + \frac{\beta}{2} \cdot 1 \\
&= \beta.
\end{align*}
\end{proof}
The condition for achieving the desired test power, given in Lemma~\ref{lem:type2}, involves two main components.: the  term~$\gamma(n,p,\beta,D,L)$, which corresponds to the sample complexity bound (see Lemma~\ref{lem:concentration}), and the random empirical quantile threshold $\widehat{c}^{B}_{1-\alpha,N}$, which constitutes the main difficulty in the theoretical analysis of permutation-based testing. To address this issue, we first rely on \citet[Lemma~6]{pmlr-v206-domingo-enrich23a}, which establishes a connection between the (random) empirical quantile threshold $\widehat{c}^{B}_{1-\alpha,N}$ and the (deterministic) quantile $c_{1-\alpha,N}$ of the permuted test statistic, conditional on the samples and the projection directions.


\begin{lemma}[Relation between population and empirical quantiles]\label{lem:connection} Let $\alpha, \beta \in (0,1)$, and let $N \ge 1$. 
Denote by $B$ the number of sampled permutations as described in Algorithm~\ref{algorithm}.
Given the datasets $\mathcal{Y}_n = (y_1,\dots,y_n)$ and $\mathcal{Z}_m = (z_1,\dots,z_m)$, whose points lie in the ball centered at the origin with radius $D$, as well as a set of projection directions $\Theta = (\theta_1,\dots,\theta_L)$ on the unit sphere $\mathbb{S}^{d-1}$, we have
\begin{align*}
\mathbb{P}_{r}\!\left( 
    \widehat{c}^{B}_{1-\alpha,N} \leq c_{1-\alpha_1,N} 
    \;\middle|\; \mathcal{Y}_n, \mathcal{Z}_m, \Theta
\right) > 1 - \tfrac{\beta}{2},
\end{align*}
where $\alpha_1 \coloneqq  \left( \dfrac{\beta/2}{\binom{B}{\lfloor \alpha(B+1)\rfloor}} \right)^{1 / \lfloor \alpha(B+1) \rfloor}.$
\end{lemma}

Consequently, obtaining an upper bound for $c_{1-\alpha_1, N}$ immediately implies an upper bound for the empirical quantile threshold $\widehat{c}^{\,B}_{1-\alpha, N}$. The next proposition establishes such an upper bound for the population quantile.


\begin{proposition}[Upper bound for $c_{1-\alpha,N}$]\label{quantilebound}
Let $p \ge 1$ and $\alpha \in (0,1)$. Under the same assumptions as in Lemma~\ref{lem:connection}, and for brevity, we denote $N\coloneqq n+m$. Moreover, without loss of generality, we assume that $n \le m$. Then the $(1-\alpha)$-quantile of the permutation distribution, $c_{1-\alpha,N}$ introduced in \eqref{eq4}, satisfies
\begin{equation}\label{bound1}
c_{1-\alpha,N}
\;\le\; \sqrt{\frac{8(2D)^{2p}\log(1/\alpha)}{3n}} + \frac{(4D)^p\sqrt{2}}{\sqrt{n}}.
\end{equation}
\end{proposition}
\begin{proof}
For readability, we denote 
\begin{equation*}
h(n,m)\coloneqq\frac{2\left(n+m-\frac{1}{2}\right)}{nm\left(\dfrac{(2D)^p}{n}+\dfrac{(2D)^p}{m}\right)^2}\cdot\left(1-\frac{1}{2\max\{m,n\}}\right).
\end{equation*}

Setting $\epsilon \coloneqq \sqrt{\ln\left(\dfrac{1}{\alpha}\right)\dfrac{1}{h(n,m)}}$, Lemma~\ref{Sensitivity} guarantees
\begin{align*}\mathbb{P}_{\pi}\left[\widehat{\operatorname{SW}}^{p,\pi}_p\ge \sqrt{\ln\left(\frac{1}{\alpha}\right)\frac{1}{h(n,m)}} + \mathbb{E}_{\pi}(\widehat{\operatorname{SW}}^{p,\pi}_p)\,\Big|\, \mathcal{Y}_n, \mathcal{Z}_m, \Theta\right]\le \alpha.
\end{align*}
It implies, by definition of $c_{1-\alpha,N}$ (recalled in Eq.~\eqref{recalquantile1}), that
\begin{align}\label{firstboundforquantile}
c_{1-\alpha,N}\le \sqrt{\ln\left(\frac{1}{\alpha}\right).\frac{1}{h(n,m)}}+\mathbb{E}_{\pi}(\widehat{\operatorname{SW}}^{p,\pi}_p).
\end{align}

We now derive a more explicit and tractable expression for $h(n, m)$. Since $1\le n \le m$, we have 
\begin{align*}
h(n,m) =\frac{n(2m-1)(n+m-\tfrac12)}{(2D)^{2p}(n+m)^2}.
\end{align*}
Setting $x\coloneqq \dfrac{m}{n}\ge 1$, we obtain
\begin{align*}
h(n,m) &= \frac{n\left(2x-\tfrac{1}{n}\right)\left(1+x-\tfrac{1}{2n}\right)}{(2D)^{2p}(1+x)^2} \\
& \ge \frac{n\left(2x-\tfrac{1}{n}\right)\left(1+x-\tfrac{1}{2n}\right)}{(2D)^{2p}(1+x)^2},
\end{align*}
where the last inequality follows from the facts that $2x - \dfrac{1}{n} \ge 2x - 1$ and $1 + x - \dfrac{1}{2n} \ge x + \dfrac{1}{2}$.

Moreover, since the function $\phi(x)=\dfrac{(2x-1)\left(x+\frac{1}{2}\right)}{\left(1+x\right)^2}$ is increasing for $x\ge 1$, it follows that
\begin{align}\label{boundforh}
h(n,m) \ge \frac{3n}{8(2D)^{2p}}.    
\end{align}
In Proposition~\ref{boundexpectation}, we have shown that, given the samples $\mathcal{Y}_n$, $\mathcal{Z}_m$ and the projection directions $\Theta=(\theta_1,\dots\theta_{L})$, the following bound holds:
\begin{align}\label{boundforexpectationterm}
\mathbb{E}_{\pi}\left[\widehat{\operatorname{SW}}^{p,\pi}_p\right]\le (4D)^{\,p} \frac{\sqrt{2}}{\sqrt{n}}.    
\end{align}
Combining inequalities~\eqref{firstboundforquantile},\eqref{boundforh}, and \eqref{boundforexpectationterm}, we obtain
\begin{equation*}
c_{1-\alpha,N}
\;\le\; \sqrt{\frac{8(2D)^{2p}\log(1/\alpha)}{3n}} + \frac{(4D)^p\sqrt{2}}{\sqrt{n}}.
\end{equation*}

\end{proof}
Having established an upper bound for $c_{1-\alpha,N}$, we now rely on Lemma~\ref{lem:connection} to derive the following proposition, which provides a bound for the empirical quantile $\widehat{c}^B_{1-\alpha,N}$.
\begin{proposition}[Bound on the empirical quantile]\label{boundempiricalquantile}
Let $\beta > 0$ and $\tfrac{1}{B+1} \le \alpha < 1$, where $B$ denotes the number of permutations as described in Algorithm~\ref{algorithm}. 
Set $\omega_{\alpha} \coloneqq \lfloor \alpha (B + 1) \rfloor$. 
Then, conditional on the samples $\mathcal{Y}_n$ and $\mathcal{Z}_m$ (with $n\le m$ and $N\coloneqq n+m)$ whose points lie in the ball centered at the origin with radius $D$, as well as on a set of projection directions $\Theta = (\theta_1, \dots, \theta_L)$ on the unit sphere $\mathbb{S}^{d-1}$, we have
\begin{equation*}
    \mathbb{P}_r\!\left(
        \widehat{c}^B_{1-\alpha,N}
        \;\le\;
\sqrt{\,\frac{8(2D)^{2p}}{3n}\,
   \log\!\left(
      \frac{2e}{\alpha \left(\tfrac{\beta}{2}\right)^{1/\omega_{\alpha}}}
   \right)} 
\;+\; \frac{(4D)^p}{\sqrt{n}}
    \right) \;>\; 1 - \frac{\beta}{2}.
\end{equation*}
\end{proposition}
\begin{proof}
Proposition~\ref{quantilebound} guarantees that
\begin{align}\label{bound1update}
c_{1-\alpha_1,N}
\;\le\; \sqrt{\frac{8(2D)^{2p}\log(1/\alpha_1)}{3n}} + \frac{(4D)^p\sqrt{2}}{\sqrt{n}},
\end{align}
where 
\begin{align*}
\alpha_1 = \left( \frac{\beta/2}{\binom{B}{\lfloor \alpha(B+1)\rfloor}} \right)^{1 / \lfloor \alpha(B+1) \rfloor}.
\end{align*}

To make this bound more explicit, we upper bound the logarithmic term. Let $\omega_{\alpha} \coloneqq  \lfloor \alpha(B+1) \rfloor \geq 1$. Then
\begin{equation}\label{bound2}
\log\!\left( \binom{B}{\omega_{\alpha}}^{1/\omega_{\alpha}} \right)
\le \log\!\left( \frac{eB}{\omega_{\alpha}} \right)
\le \log\!\left( \frac{2eB}{\alpha(B+1)} \right)
\le \log\!\left( \frac{2e}{\alpha} \right).
\end{equation}
The first inequality follows from the bound $\binom{n}{k} \le \left(\frac{en}{k}\right)^k,$ and the second from the fact that $\lfloor x \rfloor \ge \frac{x}{2}$ whenever $x \ge 1$.

Combining Eq.~\eqref{bound1update} and Eq.~\eqref{bound2} yields
\begin{align}\label{bound3}
c_{1-\alpha_1,N}
&\;\le\;
\sqrt{\,\frac{8(2D)^{2p}}{3n}\,
   \log\!\left(
      \frac{2e}{\alpha \left(\tfrac{\beta}{2}\right)^{1/\omega_{\alpha}}}
   \right)} 
\;+\; \frac{(4D)^p\sqrt{2}}{\sqrt{n}}.
\end{align}
Hence, from Lemma~\ref{lem:connection}, Eq.~\eqref{bound3}, and also conditioning on the samples $\mathcal{Y}_n$, $\mathcal{Z}_m$, and the projection directions $\theta_1,\dots,\theta_L$, we have
\begin{align*}
    \mathbb{P}_r\!\left(
        \widehat{c}^B_{1-\alpha,N}
        \;\le\;
\sqrt{\,\frac{8(2D)^{2p}}{3n}\,
   \log\!\left(
      \frac{2e}{\alpha \left(\tfrac{\beta}{2}\right)^{1/\omega_{\alpha}}}
   \right)} 
\;+\; \frac{(4D)^p\sqrt{2}}{\sqrt{n}}
    \right) \;>\; 1 - \frac{\beta}{2}.
\end{align*}
\end{proof}

Finally, Theorem~\ref{maintheorem} is proved by combining Lemma~\ref{lem:type2} and Proposition~\ref{boundempiricalquantile}, as outlined in the following.

\subsubsection*{Proof of Theorem~\ref{maintheorem}}
To ensure that the test has power at least $1 - \beta$, Lemma~\ref{lem:type2} shows it suffices to find conditions on $\mu$ and $\nu$ such that
\begin{equation}\label{result2}
\mathbb{P}_{\mu\times \nu\times r \times \sigma}\!\left[
\operatorname{SW}^p_p(\mu,\nu)\;\ge\; (2D)^p\left(
\sqrt{\frac{\log(8/\beta)}{2L}}
+ \sqrt{\frac{\log(8/\beta)}{n}}
\right) + \widehat{c}^{\,B}_{1-\alpha,N}
\right] \;>\; 1-\dfrac{\beta}{2}.
\end{equation}
From Proposition~\ref{boundempiricalquantile} and the condition provided in Eq.~\eqref{result2}, a sufficient condition for guaranteeing a power at least $1-\beta$ is given by 
\begin{align*}
\sqrt{\,\frac{8(2D)^{2p}}{3n}\,
   \log\!\left(
      \frac{2e}{\alpha \left(\tfrac{\beta}{2}\right)^{1/\omega_{\alpha}}}
   \right)} 
\;+\; \dfrac{(4D)^p\sqrt{2}}{\sqrt{n}}\le -(2D)^p\left(
\sqrt{\frac{\log(8/\beta)}{2L}}
+ \sqrt{\frac{\log(8/\beta)}{n}}
\right) + \operatorname{SW}_p^p(\mu,\nu)     
\end{align*}
or equivalently,
\begin{align*}
    \operatorname{SW}_p^p(\mu,\nu) \ge  \sqrt{\,\frac{8(2D)^{2p}}{3n}\,
   \log\!\left(
      \frac{2e}{\alpha \left(\tfrac{\beta}{2}\right)^{1/\omega_{\alpha}}}
   \right)} 
\;+\; \frac{(4D)^p\sqrt{2}}{\sqrt{n}}+(2D)^p\left(
\sqrt{\frac{\log(8/\beta)}{2L}}
+ \sqrt{\frac{\log(8/\beta)}{n}}
\right). 
\end{align*}
The above inequality implies the claimed bound on the test's power.

\section{LOWER BOUNDS}\label{sec:Lowerboundproof}

In this section, we present the proofs of the lower bounds stated in the main body, namely Proposition~\ref{minimaxmultinomial} and Proposition~\ref{lowerboundseperation}. We begin by explaining, in Section~\ref{sec:goodnesstesting}, why the problem of obtaining lower bounds for two-sample testing can be reduced to that of one-sample testing—i.e., goodness-of-fit testing. Next, we present the proofs of the two propositions.

\subsection{Reduction to goodness-of-fit testing}\label{sec:goodnesstesting}
Following \citet[Lemma~1]{arias2018remember}, we recast the task of establishing lower bounds for the two-sample testing problem as the conceptually simpler task of establishing lower bounds for the one-sample testing problem (also called goodness-of-fit). Intuitively, goodness-of-fit can be viewed as a special case of two-sample testing in which one of the distributions is fully known---which is equivalent to having access to an infinite number of samples from it.

Fix a known reference distribution $\nu_0$
and observe independent samples
$Y_1,\dots,Y_n \sim \nu$, where $\nu$ is an unknown distribution. We consider the following goodness-of-fit testing problem:
\begin{align*}
\mathcal{H}_0:\ \nu=\nu_0
\qquad\text{vs.}\qquad
\mathcal{H}_1:\ \operatorname{SW}_p^p(\nu,\nu_0)\ \ge \epsilon_n .\end{align*}

Similarly to two-sample testing (see Section~\ref{sec:minimaxoptimality}), we define the minimax separation for one-sample testing at level $\alpha$ and power $1-\beta$ as
\begin{equation*}
\epsilon_n^\dagger \coloneqq  
\inf\Biggl\{\epsilon_n > 0 :
\inf_{\Delta_n \in \Phi_{\alpha,n,\nu_0}}
\sup_{\nu \in \mathcal{P}_1(\epsilon_n,\nu_0)}
\mathbb{E}_{\nu}[1-\Delta_n] \le \beta\Biggr\},
\end{equation*}
where
\begin{itemize}
  \item $\Phi_{\alpha,n,\nu_0}$ is the set of tests $\Delta_n$, based on $n$ samples, of level at most $\alpha$, that is, $\mathbb{E}_{\nu_0}[\Delta_n]\le \alpha$;
  \item $\mathcal{P}_1(\epsilon_n,\nu_0)$ is the set of distributions $\nu$ such that $\operatorname{SW}^p_p(\nu,\nu_0) \;\ge\; \epsilon_n .$
\end{itemize}
Because the two-sample setting involves estimating both distributions from data—unlike the one-sample case, where one distribution is known exactly—the minimax separation for two-sample testing cannot be smaller than that for one-sample testing. Consequently,
\begin{align*}
\epsilon_{n,m}^\dagger \ \ge\ \epsilon_n^\dagger,
\end{align*}
where $\epsilon_n^\dagger$ and $\epsilon_{n,m}^\dagger$ denote the minimax separations for the one-sample and two-sample problems, respectively. In particular, a lower bound on the minimax separation in the one-sample setting directly implies a lower bound in the two-sample setting. 

In what follows, we prove that $\epsilon_n^\dagger \gtrsim n^{-1/2}$ which immediately yields
$\epsilon_{n,m}^\dagger \gtrsim n^{-1/2}$. 

\subsection{Proof of Proposition~\ref{minimaxmultinomial}}\label{sec:proofminimaxmultinomial}

Let $[d]\coloneqq\{1,\dots,d\}$, and denote by $\mathcal{P}^{(d)}_{\mathrm{Multi}}$ the class of multinomial distributions on $[d]$. In this section, we provide the proof of Proposition~\ref{minimaxmultinomial}, which establishes a lower bound on the minimax separation rate for the two-sample testing problem over $\mathcal{P}^{(d)}_{\mathrm{Multi}} \times \mathcal{P}^{(d)}_{\mathrm{Multi}}$.

The proof of Proposition~\ref{minimaxmultinomial} is based on Ingster’s method \citep{ingster1993asymptotically}, a classical approach for deriving minimax lower bounds in hypothesis testing. We first provide a brief overview of this technique.

\paragraph{Ingster’s method for minimax lower bounds.} Let $\nu_0$ denote a fixed null distribution. The minimax Type~II error at level $\alpha$ (for one-sample testing problem) is defined as
\begin{align*}
R^{\dagger}_{n,\alpha}
\coloneqq 
\inf_{\phi\in\Phi_{\alpha ,n,\nu_0}}
\;\sup_{\nu \in\mathcal{P}_1(\epsilon_n,\nu_0)}\,
\mathbb{P}_\nu\bigl(\phi=0\bigr),
\end{align*}
where $\Phi_{\alpha,n,\nu_0}$ and $\mathcal{P}_1(\epsilon_n,\nu_0)$ are defined in Section~\ref{sec:goodnesstesting}.

Select $T \geq 1$ distributions $\nu_1, \ldots, \nu_T\in\mathcal{P}_1(\epsilon_n, \nu_0)$, and define a mixture distribution $Q$ such that, for any measurable set $A$,
\begin{align*}
Q(A)\coloneqq \frac{1}{T}\sum_{t=1}^T \nu_t^{\otimes n}(A),   
\end{align*}
where the $n$-fold product distribution of a distribution $\nu_t$ is denoted as $\nu_t^{\otimes n}$.

Given $n$ i.i.d.\ observations $Y_1,\dots,Y_n$ drawn from an unknown distribution $\nu$, we denote the likelihood ratio between $Q$ and the null distribution $\nu_0$ by
\begin{align*}
L_n(Y_1, \dots, Y_n) \coloneqq  \frac{\mathrm{d}Q}{\mathrm{d}\nu_0^{\otimes n}} (Y_1, \dots, Y_n) 
\;=\;\frac{1}{T}\sum_{t=1}^T
\prod_{i=1}^n \frac{\nu_t(Y_i)}{\nu_0(Y_i)},
\end{align*}
where, with a slight abuse of notation, we denote by $\nu_0$ and $(\nu_t)_{t=1}^T$ both the probability measures and their densities with respect to the Lebesgue measure. 

Given these definitions, the following lemma provides a condition on the likelihood ratio between the mixture distribution $Q$ and the null distribution $\nu_0$ to obtain a lower bound for the minimax Type~II error $R^{\dagger}_{n,\alpha}$.
\begin{lemma}[Lower bound via a mixture]\label{lem:ingster}
Fix $\alpha\in(0,1)$ and $\beta\in(0,1-\alpha)$. If
\begin{align*}
\mathbb{E}_{\nu_0}\!\left[L_n^2\right]
\;\le\; 1+4(1-\alpha-\beta)^2,
\end{align*}
then $R^{\dagger}_{n,\alpha}\;\ge\;\beta$.
\end{lemma}
The proof of this lemma can be found in \citet[Section~11]{wasserman_minimax_theory} and in \citet[Appendix~H]{kim2022minimax}.

To apply Lemma~\ref{lem:ingster}, we need to (i) specify the distributions $Q$ and $\nu_0$, (ii) compute the expectation of the squared likelihood ratio under the null. We borrow the construction of perturbed distributions around the uniform measure on $[d]$ from \citet[Appendix~H]{kim2022minimax}, and adapt the corresponding computations to the sliced Wasserstein distance setting.

Let $\nu_0$ denote the uniform distribution over the set $[d]\coloneqq \{1, \dots, d\}$, that is, $\nu_0(k) = \tfrac{1}{d}$ for $k=1,\dots,d$. 

Define the set
\begin{equation} 
\mathcal{M}_d \coloneqq \{\eta \in\{-1, 1\}^d:\displaystyle \sum_{k=1}^{d} \eta_k = 0\}. 
\end{equation} 
We may assume without loss of generality that $d$ is even, as the proof for the odd-dimensional case follows analogously by setting the last coordinate to zero $\eta_d = 0$. 
Given $\eta \in \mathcal{M}_d$, we define the measure $\nu_{\eta}$ as 
\begin{align*}
    \nu_{\eta}(k)\coloneqq \nu(k) + \frac{2\epsilon_{n}}{d}\eta_{k}, \quad \text{ for } k=1,\dots,d,
\end{align*}
for some $\epsilon_n \leq \tfrac{1}{2}$ that we will specify later in the proof.

Note that, for any $\eta \in \mathcal{M}_d$, the distribution $\nu_{\eta}$ is a valid probability distribution on the set $[d]$. Indeed,
\begin{enumerate} 
\item $\nu_{\eta}$ is non negative.
\item $\displaystyle \sum_{k=1}^{d} \nu_\eta(k)= \sum_{k=1}^{d}\left(\nu(k) + \frac{2\epsilon_{n}}{d}\eta_{k}\right)=\sum_{k=1}^d\nu(k)+\frac{2\epsilon_n}{d}\sum_{k=1}^d\eta_{k} = 1.$
\end{enumerate} 

Let us now evaluate the sliced Wasserstein distance between $\nu_0$ and its perturbation $\nu_\eta$.
By construction, exactly $d/2$ coordinates $k$ satisfy $\eta_k = 1$, and the remaining $d/2$ coordinates satisfy $\eta_k = -1$. For coordinates with $\eta_k = 1$, the weight in $\nu_{\eta}$ is increased by $\frac{2\epsilon_n}{d}$ compared to $\nu_0$, while for coordinates with $\eta_k = -1$, the weight is decreased by the same amount. Hence, the total excess mass 
between the $d/2$ heavier points and the $d/2$ lighter points is 
\begin{align*}
\frac{d}{2} \cdot \frac{2\epsilon_n}{d} = \epsilon_n. \end{align*}
Since $\nu_0$ is uniform over $[d]$ and the distance between any two distinct points in $[d]$ is at least $1$, moving any unit of excess mass incurs a cost of at least $1^2 = 1$. Therefore, the total transportation cost is at least $\epsilon_n$.
Since both $\nu_0$ and $\nu_{\eta}$ are one-dimensional distributions, 
the sliced Wasserstein distance coincides with the standard Wasserstein distance and 
\begin{align*}
    \operatorname{SW}^p_p(\nu_0, \nu_\eta) = \operatorname{W}_p^p(\nu_0, \nu_\eta) \ge \epsilon_n.
\end{align*}

Let $T$ be the cardinality of $\mathcal{M}_d$.
In order to apply Ingster's method, we introduce the uniform mixture $\nu_{\eta(1)},\dots,\nu_{\eta(T)}$,
\begin{align*}
Q\coloneqq \frac{1}{T}\sum_{\eta \in \mathcal{M}_d}^{T}\nu_{\eta}.
\end{align*} 
The likelihood ratio between $Q$ and the null distribution $\nu_0$ will be 
\begin{align*}
L_{n}(Y_1, \dots, Y_n)\coloneqq \frac{1}{T}\sum_{\eta \in \mathcal{M}_d}\prod_{i=1}^{n}\frac{\nu_{\eta}(Y_i)}{\nu_0(Y_i)}.
\end{align*}

With those ingredients on hand, we are now ready to compute the expected value of the squared likelihood ratio. The subsequent computation follows exactly the same steps as in \citet[Appendix~H]{kim2022minimax}. We present it here for completeness.
\begin{align*} L_{n}^2(Y_1, \dots, Y_n) 
&= \frac{1}{T^2} \sum_{\eta, \eta' \in \mathcal{M}_d} \prod_{i=1}^{n} \frac{\nu_{\eta}(Y_i) \, \nu_{\eta'}(Y_i)}{\nu_0(Y_i)^2} \\[0.3em] 
&= \frac{1}{T^2} \sum_{\eta, \eta' \in \mathcal{M}_d} \prod_{i=1}^{n} \frac{(\frac{1}{d} + \frac{2\epsilon_n}{d} \eta_{Y_i}) (\frac{1}{d} + \frac{2\epsilon_n}{d}  \eta'_{Y_i})} {1/d^2} \\[0.3em] 
&= \frac{1}{T^2} \sum_{\eta, \eta' \in \mathcal{M}_d} \prod_{i=1}^{n} (1 + 2\epsilon_n \eta_{Y_i})(1 + 2\epsilon_n \eta'_{Y_i}). 
\end{align*}
By definition of $\eta \in \mathcal{M}_d$, taking the expectation over $Y_1,\dots,Y_n \sim \nu_0$, we obtain
\begin{align*} \mathbb{E}_{\nu_0}\!\left[ L_{n}^2 \right] &= \frac{1}{T^2} \sum_{\eta, \eta' \in \mathcal{M}_d} \left( 1 + \frac{4\epsilon_n^2}{d} \sum_{k=1}^d \eta_{k} \eta'_{k} \right)^{n}.
\end{align*} 
Moreover, using the inequality $1+x \le e^{x}$, which holds for any $x \in \mathbb{R}$,
\begin{align*} \mathbb{E}_{\nu_0}\!\left[ L_{n}^2 \right] \le \frac{1}{T^2} \sum_{\eta, \eta' \in \mathcal{M}_d} \exp\!\left( \frac{4 n\epsilon_n^2}{d} \sum_{k=1}^d \eta_{k} \eta'_{k} \right).
\end{align*} 
Let $\eta$ and $\eta^*$ be independent random variables uniformly distributed on $\mathcal{M}_d$. We have 
\begin{align*} 
\frac{1}{T^2} \sum_{\eta, \eta' \in \mathcal{M}_d} \exp\!\left( \frac{4n\epsilon_n^2}{d} \sum_{k=1}^d \eta_{k} \eta'_{k} \right) &= \mathbb{E}_{\eta, \eta^*} \!\left[ \exp\!\left( \frac{4n\epsilon_n^2}{d} \langle \eta, \eta^* \rangle \right) \right]. 
\end{align*} 
Applying Lemma~2 in~\cite{dubhashi1996balls}, we obtain \begin{align*} \mathbb{E}_{\nu_0}\!\left[ L_{n}^2 \right] &\le \mathbb{E}_{\eta, \eta^*} \!\left[ \exp\!\left( \frac{4n\epsilon_n^2}{d} \langle \eta, \eta^* \rangle \right) \right] \\ &\le \prod_{k=1}^d \mathbb{E}_{\eta_k, \eta_k^*} \!\left[ \exp\!\left( \frac{4n\epsilon_n^2}{d} \eta_k \eta_k^* \right) \right] \\ &= \prod_{k=1}^d \cosh\!\left( \frac{4n\epsilon_n^2}{d} \right) \\ &\le \exp\!\left( \frac{d}{2} \left( \frac{4n\epsilon_n^2}{d} \right)^2 \right) = \exp\!\left( \frac{8n^2\epsilon_n^4}{d} \right), \end{align*} where the last inequality follows from $\cosh(x) \le e^{x^2/2}$, which holds for any $x \in \mathbb{R}$.

The proof is concluded by observing that Lemma~\ref{lem:ingster} guarantees that the minimax Type~II error is lower bounded by $\beta$ provided that
\begin{align*}
    \epsilon_n \le \frac{1}{\sqrt{n}}\sqrt[4]{\dfrac{d\log\left[1+4(1-\alpha-\beta)^2\right]}{8}}.
\end{align*}

\subsection{Proof of Proposition~\ref{lowerboundseperation}}
Proposition~\ref{minimaxmultinomial} states a lower bound on the minimax separation rate for the two-sample testing problem over the class of multinomial distributions. Turning to the more general class of distributions with bounded support, denoted by $\mathcal{P}_{\mathbb{R}^d}(D)$ for some $D > 0$, the proof of Proposition~\ref{lowerboundseperation} establishes the corresponding lower bound on the minimax separation rate for the two-sample testing problem over this class.

We begin by presenting some preliminary ingredients required for the proof of Proposition~\ref{lowerboundseperation}.
\subsubsection*{Main ingredients for the proof}
We begin this section by recalling a basic fact: the uniform distribution on the sphere is invariant under orthogonal transformations, as stated below.
\begin{lemma}[Rotation invariance]\label{rotationinvariance}
If $U$ is uniform on $\mathbb{S}^{d-1}$ and $Q$ is an orthogonal matrix, then $QU$ has the same distribution as $U$.
\end{lemma}
\begin{proof}
Let $\sigma$ denote the $(d-1)$-dimensional surface area measure on $\mathbb{S}^{d-1}$. Orthogonal maps preserve surface area on the sphere, that is, $\sigma(QA)=\sigma(A)$ for every Borel set $A \subset \mathbb{S}^{d-1}$. Hence, for any Borel set $A$, we have
$$
\mathbb{P}\left(QU \in A\right) = \mathbb{P}(U \in Q^{-1}A)=\frac{\sigma(Q^{-1}A)}{\sigma(\mathbb{S}^{d-1})}=\frac{\sigma{A}}{\sigma(\mathbb{S}^{d-1})}=\mathbb{P}\left(U \in A\right).
$$
Therefore $QU$ and $U$ have the same law.
\end{proof}
We now identify the law of a single coordinate under the uniform distribution on the sphere.

\begin{lemma}[First coordinate of a spherical uniform distribution]\label{firstcoordinate}
Let $U=(U_1,\dots,U_d)\sim\operatorname{Unif}(\mathbb S^{d-1})$.
\begin{enumerate}
\item For $d=1$, $U_1^2\equiv 1$ a.s. (a degenerate distribution at $1$).
\item For $d\ge 2$, $U_1^2 \sim \operatorname{Beta}\!\left(\dfrac12,\dfrac{d-1}{2}\right)$.
\end{enumerate}
\end{lemma}

\begin{proof}
\paragraph{Case $d=1$.} It is clear that $\mathbb S^0=\{\pm1\}$ and $U_1=\pm1$ with probability $1/2$ each. Therefore, $U_1^2\equiv 1$ almost surely.
\paragraph{Case $d\geq 2$.} Let $Z=(Z_1,\dots,Z_d)$ be standard normal in $\mathbb R^d$ (i.e., $Z\sim\mathcal N(0,I_d)$). By Equation~(3.15) in \cite{vershynin2018high}, $U\coloneqq Z/\|Z\|\sim\operatorname{Unif}(\mathbb S^{d-1})$.
Since $Z_1\sim\mathcal N(0,1)$ and the $Z_i$'s are independent,
$$
X\coloneqq  Z_1^2\sim\chi^2_1=\Gamma\!\left(\dfrac12;2\right),
\qquad
Y\coloneqq \sum_{i=2}^d Z_i^2\sim\chi^2_{d-1}=\Gamma\!\left(\dfrac{d-1}{2};2\right),
$$
with $X\perp Y$. 

Finally, using the classical fact that if $M \sim \chi^2(m)$ and $N \sim \chi^2(n)$ are independent, then
$$
\dfrac{M}{M+N} \sim \operatorname{Beta}\!\left(\dfrac{m}{2}, \dfrac{n}{2}\right),
$$
(see, e.g., \citet[Section~3.9]{StatProofBook2024}), we conclude that
$$
U_1^2 \sim \operatorname{Beta}\!\left(\dfrac{1}{2}, \dfrac{d-1}{2}\right).
$$
As a consequence, for $p \geq 1$, and following the result of \citet[Page~145]{balakrishnan2004primer}, we have
\begin{align*}
\mathbb E\,|U_1|^p \;=\;
\begin{cases}
\dfrac{\Gamma\!\left(\tfrac{p+1}{2}\right)\Gamma\!\left(\tfrac d2\right)}
{\sqrt{\pi}\,\Gamma\!\left(\tfrac{p+d}{2}\right)}, & d\ge 2,\\[1.25em]
1, & d=1,
\end{cases}
\end{align*}
where $\Gamma(\cdot)$ denotes the classical Gamma function.
\end{proof}
We recall two classical inequalities that are particularly useful for bounding Gamma functions: Wendel’s inequality and Stirling’s inequalities.
\begin{lemma}[Wendel's inequality]\label{Wendelinequality}
We denote by $\Gamma(\cdot)$ the classical Gamma function. For $0\le a<1$ and $x>0$,  
    $$\left(\frac{x}{x+a}\right)^{1-a} \le \frac{\Gamma(x+a)}{x^{a}\Gamma(x)} \le 1.
    $$
\end{lemma}
\begin{lemma}[Stirling's inequality]\label{Stirlinginequality}
We denote by $\Gamma(\cdot)$ the classical Gamma function. For $x>0$,  
    $$
    \sqrt{2\pi}\,x^{\,x-\tfrac{1}{2}} e^{-x} 
    \;\le\; \Gamma(x) 
    \;\le\; \sqrt{2\pi}\,x^{\,x-\tfrac{1}{2}} e^{-x} e^{\tfrac{1}{12x}}.
    $$
\end{lemma}
Proofs of Lemmas~\ref{Wendelinequality} and~\ref{Stirlinginequality} can be found in \citet{wendel1948note} and \citet[Theorem~1]{jameson2015simple}, respectively.

We then close these preliminaries with a simple yet essential fact from linear algebra.
\begin{lemma}\label{existrotation}
For any unit vector $u \in \mathbb{R}^d$ with $\|u\|=1$, there exists a rotation $Q \in SO(d)$ such that $Q e_1 = u$, where $e_1 = (1,0,\dots,0)^\top$ and $SO(d) = \{Q \in \mathbb{R}^{d \times d} : Q^\top Q = I,\ \det Q = 1\}$.
\end{lemma}
\begin{proof}
Extend $\{u\}$ to a basis of $\mathbb{R}^d$ and apply the Gram--Schmidt procedure to obtain an orthonormal basis $\{u, v_2, \dots, v_d\}$.
Let
$$
Q \coloneqq  [\,u \mid v_2 \mid \cdots \mid v_d\,] \in \mathbb{R}^{d \times d}.
$$
Then the columns of $Q$ are orthonormal and $Q^\top Q = I$. In particular, $Q$ is orthogonal and satisfies $Q e_1 = u$.

If $\det(Q)=1$, then $Q \in SO(d)$ and we are done. If $\det(Q)=-1$, replace one column $v_j$ with $-v_j$ for some $j \ge 2$ (e.g., $j=2$). The resulting matrix
$$
\widetilde{Q} = [\,u \mid -v_2 \mid v_3 \mid \cdots \mid v_d\,]
$$
remains orthogonal, still satisfies $\widetilde{Q} e_1 = u$, and has $\det(\widetilde{Q}) = -\det(Q) = 1$.
Therefore, there exists $Q \in SO(d)$ such that $Q e_1 = u$.
\end{proof}
With these preparatory results in place, we are now ready to prove Proposition~\ref{lowerboundseperation}.
\subsubsection{Formal proof}
As discussed in Section~\ref{sec:goodnesstesting}, once a lower bound for the minimax rate  $\epsilon_{n}^\dagger$ for goodness-of-fit testing is established, a lower bound  $\epsilon_{n,m}^\dagger$ for two-sample testing follows immediately. For the general case of distributions with bounded support, we apply Le Cam two-point method \citep{lecam1973convergence, le2012asymptotic}, a classical techniques for deriving minimax lower bounds.

Let $\nu_0$ be a reference distribution. Let $\epsilon_n > 0$ and let $\mu_0$ be an alternative distribution in $\mathcal{P}_1(\epsilon_n,\nu_0)$. Using the same notation as for the one-sample testing problem (see Section~\ref{sec:goodnesstesting}), we obtain
\begin{align*}
\inf_{\Delta_n \in \Phi_{\alpha,n,\nu_0}} \sup_{\mu \in \mathcal{P}_1(\epsilon_n,\nu_0)} 
   \mathbb{E}_{\mu}[1-\Delta_n] 
&\ge \inf_{\Delta_n \in \Phi_{\alpha,n,\nu_0}} \mathbb{E}_{\mu_0}[1-\Delta_n] \\
&= 1 - \sup_{\Delta_n \in \Phi_{\alpha,n,\nu_0}} \mathbb{E}_{\mu_0}[\Delta_n] \\
&= 1 - \sup_{\Delta_n \in \Phi_{\alpha,n,\nu_0}}
   \left\{ \mathbb{E}_{\mu_0}[\Delta_n] - \mathbb{E}_{\nu_0}[\Delta_n] 
          + \mathbb{E}_{\nu_0}[\Delta_n]\right\} \\
&\ge 1 - \alpha - \sup_{\Delta_n \in \Phi_{\alpha,n,\nu_0}}
   \left\{ \mathbb{E}_{\mu_0}[\Delta_n]-\mathbb{E}_{\nu_0}[\Delta_n]\right\} \\
&\ge 1 - \alpha - d_{\mathrm{TV}}\!\left(\mu_0^{\otimes n}, \nu_0^{\otimes n}\right) \\
&\overset{(i)}{\geq} 1 - \alpha - 1 + \tfrac{1}{2} 
   e^{-D_{\mathrm{KL}}\!\left(\mu_0^{\otimes n} \,\|\, \nu_0^{\otimes n}\right)} \\
&\overset{(ii)}{=} \tfrac{1}{2} e^{-n \, D_{\mathrm{KL}}(\mu_0 \,\|\, \nu_0)} - \alpha,
\end{align*}
where $(i)$ follows from the Bretagnolle–Huber inequality introduced in \citet[Lemma~3]{canonne2022short}, and  
$(ii)$ uses the chain rule for the KL divergence. We recall that $\mu_0^{\otimes n}$ and $\nu_0^{\otimes n}$ are the $n$-fold product distributions of $\mu_0$ and $\nu_0$, respectively.

Consequently, the minimax Type~II error is at least $\beta$, that is,
\begin{align*}
\inf_{\Delta_n \in \Phi_{\alpha,n,\nu_0}} \sup_{\mu \in \mathcal{P}_1(\epsilon_n,\nu_0)} 
   \mathbb{E}_{\mu}[1-\Delta_n] \ge \beta,
\end{align*}
provided that
\begin{align*}
\alpha + \beta < 0.5
\quad \text{and} \quad
D_{\mathrm{KL}}(\mu_0 \,\|\, \nu_0) \;\le\; \frac{1}{n} \log\!\left(\frac{1}{2(\alpha+\beta)}\right).
\end{align*}
In order to obtain a lower bound, we need to choose $\mu_0,\nu_0\in \mathcal{P}_{\mathbb{R}^d}(D)$ such that
\begin{equation}\label{requirement2}
D_{\mathrm{KL}}(\mu_0 \,\|\, \nu_0) \;\le\; \frac{1}{n} \log\!\left(\frac{1}{2(\alpha+\beta)}\right).
\end{equation}
and
\begin{equation}\label{requirement1}
\operatorname{SW}^p_p(\mu_0,\nu_0) \ge \epsilon_n. 
\end{equation}

Inspired by the works of \citet[Section~E.10.1]{kim2023differentially} and adapting their work to the case of the sliced Wasserstein distance, we choose
$$
\mu_0 = p_0 \delta_x + (1-p_0)\delta_v,
\quad 
\nu_0 = q_0 \delta_x + (1-q_0)\delta_v,
$$
where $x,v \in \mathbb{R}^d$ with $\|x-v\|_2 = D$, $0 \le p_0,q_0 \le  1$, and $\delta_x$ denotes the Dirac measure at $x$. Furthermore, we also choose  
\begin{align*}
p_{0} = \dfrac{1}{2} + \min \left\{ \sqrt{\dfrac{1}{4n} 
    \log \!\left(\dfrac{1}{2(\alpha+\beta)}\right)}, \; \dfrac{1}{2} \right\},
\quad
q_{0} = \dfrac{1}{2}.
\end{align*}
With those parameters, we have
\begin{align*}
D_{\mathrm{KL}}(\mu_0 \,\|\, \nu_0)  \le  \frac{(p_0-q_0)^2}{q_0(1-q_0)}=\min\left\{\frac{1}{n}\log\left(\frac{1}{2(\alpha+\beta)}\right),1\right\} \le \frac{1}{n}\log\left(\frac{1}{2(\alpha+\beta)}\right),
\end{align*}
where the first inequality follows directly from a classical inequality between the KL and $\chi^2$ divergences (see, e.g., \citet[Lemma~2.7]{tsybakov2008introduction}).

Moreover, we have
\begin{align*}
 \operatorname{SW}^p_p(\mu_0,\nu_0)
&= \int_{\mathbb{S}^{d-1}}
\operatorname{W}_p^p\!\Big( \Pi^\theta_{\#}\mu_0,\,
\Pi^\theta_{\#}\nu_0 \Big)\,
\sigma(d\theta)   \\
&\overset{(i)}{=} \int_{\mathbb{S}^{d-1}} \left|p_0-q_0\right| \left|\langle x-v,\theta\rangle\right|^p\sigma(d\theta)\\
&\overset{(ii)}{=} \left|p_0 - q_0\right|\|x - v\|^{p}\int_{\mathbb{S}^{d-1}}\left|\langle u,\theta\rangle\right|^p \sigma(d\theta),
\end{align*}
where (i) follows from the computation of the Wasserstein distance between two Dirac measures in one dimension (see, e.g., \citet[Example~1.4]{chewi2024statistical}), and (ii) from setting 
$u = \frac{x - v}{\|x - v\|}.$

We can further develop this expression as 
\begin{align*}
 \operatorname{SW}^p_p(\mu_0,\nu_0)
& = \left|p_0 - q_0\right|\|x - v\|^{p} \mathbb{E}_{\theta \sim \operatorname{Unif}(\mathbb{S}^{d-1})}\left[\left|\langle u,\theta\rangle\right|^p\right]\\
& \overset{(iii)}{=} \left|p_0 - q_0\right|\|x - v\|^{p} \mathbb{E}_{\theta \sim \operatorname{Unif}(\mathbb{S}^{d-1})}\left[\left|\langle Qe_1,\theta\rangle\right|^p\right]\\
& = \left|p_0 - q_0\right|\|x - v\|^{p} \mathbb{E}_{\theta \sim \operatorname{Unif}(\mathbb{S}^{d-1})}\left[\left|\langle e_1,Q^\top\theta\rangle\right|^p\right]\\
&\overset{(iv)}{=} \left|p_0 - q_0\right|\|x - v\|^{p} \mathbb{E}_{\theta \sim \operatorname{Unif}(\mathbb{S}^{d-1})}\left[\left|\langle e_1,\theta\rangle\right|^p\right]\\
& = \left|p_0 - q_0\right|\|x - v\|^{p} \mathbb{E}_{\theta \sim \operatorname{Unif}(\mathbb{S}^{d-1})}\left[\left|\theta_1\right|^p\right]\\
&\overset{(v)}{=} |p_0 - q_0|\,\|x - v\|^{p}
\begin{cases}
\dfrac{\Gamma\!\left(\tfrac{p+1}{2}\right)\Gamma\!\left(\tfrac d2\right)}
{\sqrt{\pi}\,\Gamma\!\left(\tfrac{p+d}{2}\right)}, & d\ge 2,\\[1.25em]
1, & d=1,
\end{cases}
\end{align*}
where ($iii$) follows from Lemma~\ref{existrotation}, ($iv$) from Lemma~\ref{rotationinvariance}, and ($v$) from Lemma~\ref{firstcoordinate}.

We now derive lower bound for the term 
\begin{align*}
c_{d,p}\coloneqq \dfrac{\Gamma\!\left(\tfrac{p+1}{2}\right)\Gamma\!\left(\tfrac d2\right)}
{\sqrt{\pi}\,\Gamma\!\left(\tfrac{p+d}{2}\right)}.
\end{align*}

The term $\Gamma\left(\tfrac{p+1}{2}\right)$ can be controled by applying Stirling's lower bound (recalled in Lemma~\ref{Stirlinginequality}). In particular, we have
$$
\Gamma\!\left(\dfrac{p+1}{2}\right)
\;\ge\; \sqrt{2\pi}\,\Big(\dfrac{p+1}{2}\Big)^{\tfrac{p}{2}} e^{-\tfrac{(p+1)}{2}}.
$$

We next control the remaining term.

Write $\tfrac{p}{2}=m+r$ with $m=\lfloor p/2 \rfloor \in \mathbb{N}$ and $r\in[0,1)$. We then can decompose the remaining terms as follow:
$$
\frac{\Gamma(\tfrac d2)}{\Gamma(\tfrac d2+m+r)}
= \frac{\Gamma(\tfrac d2)}{\Gamma(\tfrac d2+m)} \cdot 
  \frac{\Gamma(\tfrac d2+m)}{\Gamma(\tfrac d2+m+r)}.
$$
By the recursion formula of the gamma function, $\Gamma(p+1)=p\Gamma(p)$ for $p\ge 0$, we obtain
\begin{align*}
\frac{\Gamma(\tfrac d2)}{\Gamma(\tfrac d2+m)}
\;\ge\; (\tfrac d2+m)^{-m}.
\end{align*}
For the second factor, since $0\le r < 1$, Wendel’s inequality (recalled in Lemma~\ref{Wendelinequality}) yields
$$
\frac{\Gamma(\tfrac d2+m)}{\Gamma(\tfrac d2+m+r)} \;\ge\; 
\big(\tfrac d2+m+r\big)^{-r}.
$$
Hence,
$$
\frac{\Gamma(\tfrac d2)}{\Gamma(\tfrac{d+p}{2})}=\frac{\Gamma(\tfrac d2)}{\Gamma(\tfrac{d}{2}+m+r)}
\;\ge\; (\tfrac d2+m)^{-m}\,(\tfrac d2+m+r)^{-r}
\;\ge\; \Big(\tfrac{d+p}{2}\Big)^{-\tfrac{p}{2}}.
$$
Combining the above estimates gives
$$
c_{d,p} \;\ge\; \dfrac{\sqrt{2\pi}}{\sqrt{\pi}}\,
\Big(\dfrac{p+1}{2}\Big)^{\tfrac{p}{2}} e^{-\tfrac{(p+1)}{2}}
\cdot \Big(\dfrac{d+p}{2}\Big)^{-\tfrac{p}{2}}=\sqrt{\dfrac{2}{e}}\left[\dfrac{p+1}{e(d+p)}\right]^{\tfrac{p}{2}}.
$$
As a result, we have
\begin{align*}
\operatorname{SW}^p_p(\mu_0,\nu_0)\ge \frac{D^p}{\sqrt{2e}}\left[\frac{p+1}{e(d+p)}\right]^{\tfrac{p}{2}} \min \left\{ \sqrt{\dfrac{1}{n}\log\!\left(\dfrac{1}{2(\alpha+\beta)}\right)}, \; 1 \right\}.
\end{align*}
Setting $\epsilon_n\coloneqq  \frac{D^p}{\sqrt{2e}}\left[\frac{p+1}{e(d+p)}\right]^{\tfrac{p}{2}} \min \left\{ \sqrt{\dfrac{1}{n}\log\!\left(\dfrac{1}{2(\alpha+\beta)}\right)}, \; 1 \right\}$, conditions \eqref{requirement2} and \eqref{requirement1} hold, and thus the minimax separation satisfies
\begin{align*}
\epsilon_{n}^\dagger 
\;\ge\;  \epsilon_n = 
\frac{D^p}{\sqrt{2e}}
\left[ \frac{p+1}{e(d+p)} \right]^{\tfrac{p}{2}}
\min \left\{ 
  \sqrt{\dfrac{1}{n}\log\!\Bigl(\dfrac{1}{2(\alpha+\beta)}\Bigr)}, \; 1 
\right\}.
\end{align*}
The above bound implies the bound stated in Proposition~\ref{lowerboundseperation}.
\section{PERMUTATION APPROACH}\label{sec:PermutationApproach}
In this section, we present the motivation for employing a permutation-based strategy in the two-sample testing problem using the sliced Wasserstein distance.

In hypothesis testing, the determination of the critical value plays a central role, as it directly governs the decision to reject the null hypothesis. In general, existing approaches for determining the critical value can be grouped into two main categories: asymptotic and non-asymptotic methods.

\paragraph{Asymptotic method.} In the asymptotic approach, the critical value is determined from the limiting distribution of the test statistic under the null hypothesis (see, e.g., \cite{shekhar2022permutation, shekhar2023permutation, zaremba2013b}). Specifically, to control the Type~I error in the asymptotic regime, the critical value is chosen as the $(1-\alpha)$ quantile of the null distribution. 

However, this asymptotic approach is not without limitations. In many cases, the limiting null distribution is either intractable or lacks a convenient closed form. For instance, \citet[Theorem~12]{gretton2012mmd} show that the null distribution of the squared empirical biased estimator of MMD is an infinite weighted sum of independent $\chi^2$ random variables, with weights given by the eigenvalues of the kernel operator. This makes the distribution highly dependent on the kernel choice and analytically difficult to handle. Likewise, more general tests based on U-statistics face similar challenges in the multinomial setting: the shape of the null distribution is intricately tied to the probabilities defining the underlying multinomial structure (see \citet[Figure~1]{kim2022minimax}).

The test statistic $\widehat{\operatorname{SW}}^p_p$ defined in \eqref{teststatistic} faces the same challenge. To gain intuition about its null distribution, we approximate it empirically through simulation.
Specifically, we consider the case $p=2$ and generate samples from three pairs of identical distributions with $n = m = 8000$ and $L = n = 8000$ projection directions. Repeating this procedure $4000$ times produces the histograms shown in Figure~\ref{fig:nulldistribution}. The resulting distributions vary significantly across underlying data distributions, which are unknown in practice, making it difficult to compute the critical threshold analytically.

\begin{figure}[H]
  \centering
  \begin{subfigure}{0.325\textwidth}
    \centering
    \includegraphics[width=\linewidth]{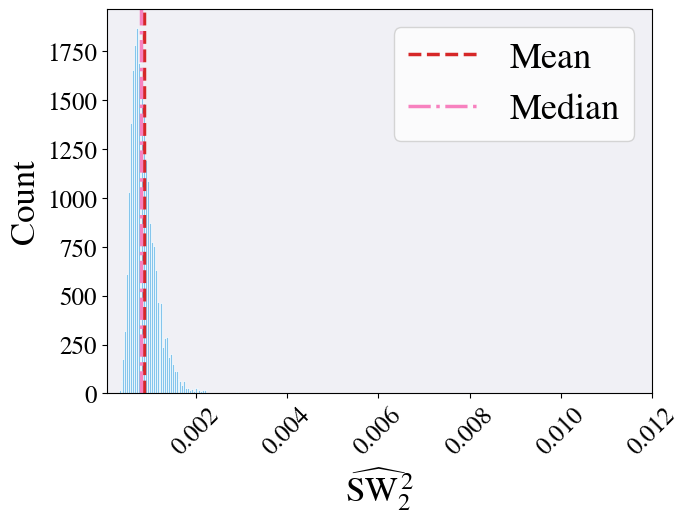}  \end{subfigure}\hspace{0.0005\textwidth}
  \begin{subfigure}{0.325\textwidth}
    \centering
    \includegraphics[width=\linewidth]{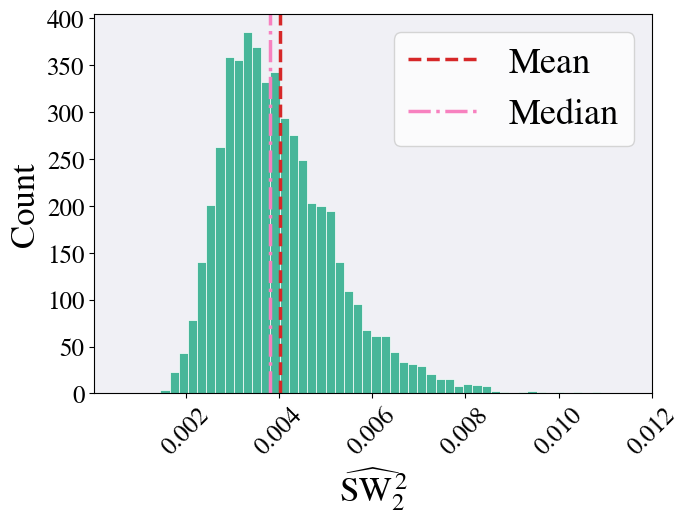}  \end{subfigure}\hspace{0.0005\textwidth}
  \begin{subfigure}{0.325\textwidth}
    \centering
    \includegraphics[width=\linewidth]{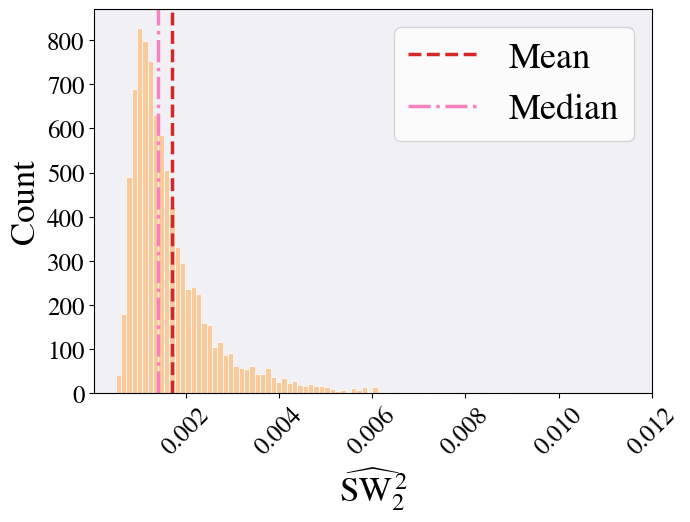}
  \end{subfigure}
  \caption{Histograms of the test statistic 
$\widehat{\operatorname{SW}}^{\,2}_{2}$ computed from $8000$ i.i.d. samples drawn
respectively from (left) Gaussian $\mathcal{N}(0,I_2)$, (middle) Uniform on
$[-1,1]^2$, and (right) a two-component Gaussian mixture 
$\tfrac{1}{2}\mathcal{N}(0,I_2)+\tfrac{1}{2}\mathcal{N}(\mathbf{m},I_2)$, 
where $\mathbf{m}=(2.5,2.5)^\top$.}
  \label{fig:nulldistribution}
\end{figure}

Recent work by \cite{rodriguez2025improved} established a Central Limit Theorem (CLT) that characterizes the asymptotic distribution of the empirical Sliced Wasserstein distance. Under suitable regularity conditions, the theorem states that: As $n, m=m(n)\to +\infty$, we obtain that
\begin{equation}\label{eq:asymptoticsliced}
\sqrt{\dfrac{k\frac{nm}{n+m}}{k+\frac{nm}{n+m}}}\left({\operatorname{\widehat{SW}}}^p_p-\operatorname{SW}^p_p(\mu,\nu)\right) \xrightarrow{d} \mathcal{N}\bigl(0,(1-\tau)w^2_{\mu,\nu}+\tau\left((1-\lambda)v^2_{\mu,\nu}+\lambda v^2_{\nu,\mu}\right)\bigr).
\end{equation}

Here, $\xrightarrow{d}$ denotes convergence in distribution, $n$ and $m = m(n)$ are the respective sample sizes, and $k = k(n)$ denotes the number of projection directions, and the parameters $\tau$ and $\lambda$ are asymptotic limits defined by
$$
\tau = \lim_{n \to +\infty} \frac{k}{k + \frac{nm}{n+m}}
\quad \text{and} \quad
\lambda = \lim_{n \to +\infty} \frac{n}{n+m}.
$$

Moreover, the two components of the asymptotic variance are defined as\footnotemark
$$w_{\mu,\nu}^2 \coloneqq  \displaystyle \int_{\mathbb{S}^{d-1}} \operatorname{W}_p^{\,2p}\!\bigl(\Pi^\theta_{\#}\mu,\Pi^\theta_{\#}\nu\bigr)\, d\sigma(\theta) - \operatorname{SW}_p^{\,2p}\!\bigl(\mu,\nu\bigr),$$
and
$$v_{\mu,\nu}^2 \coloneqq \displaystyle \int_{\mathbb{S}^{d-1}} \int_{\mathbb{S}^{d-1}} \operatorname{Cov}_{\mu}\!\bigl(\phi^{\theta},\, \phi^{\eta}\bigr)\,d\sigma(\theta)\, d\sigma(\eta),$$

where $\phi^{\theta}$ is any c-concave optimal transport potential from $\Pi^\theta_{\#}\mu$ to $\Pi^\theta_{\#}\nu$.

\footnotetext{We use the same notation as in the definition of the sliced Wasserstein distance presented in Subsection \ref{subsec:SlicedWasserStein}}

As can be observed, the asymptotic variances $w^2_{\mu,\nu}$ and $\nu^2_{\mu,\nu}$ are hard to compute and depend on the underlying unknown distributions. Furthermore, those asymptotic variances are equal to zero under the null hypothesis. As a consequence, it is not possible to construct a meaningful test based on this asymptotic null distribution. More broadly, research on the limiting distribution of the sliced Wasserstein distance remains limited, which makes it difficult to address testing problems through asymptotic analysis.

Finally, the number of observations is often limited due to economical or biological constraints (see, e.g., the discussion about applications in neurosciences in \citet[Section~0.4.1]{albert2015tests}). Consequently, asymptotic procedures may not be suitable in such small sample size settings. This is one of the reasons why non-asymptotic methods have been developed.

\paragraph{Non-asymptotic methods.} 
A classical approach is to determine the threshold of a test using a concentration inequality or a tail bound for the test statistic under the null hypothesis (see, e.g., \citet[Section~4.2]{gretton2012mmd}, \cite{wei2011dvoretzky}).
This non-asymptotic approach avoids dependence on asymptotic approximations and is therefore less sensitive to distributional assumptions. On the downside, the resulting bounds are usually conservative, leading to Type~I errors much smaller than $\alpha$ and lower power. Moreover, many of these guarantees involve constants that are not explicitly determined in theory, making the resulting thresholds difficult to compute in practice (cf. \citet[Section~III]{wang2021linear} and the illustration of \citet[Section~9]{kim2022minimax}).

There exist other data-driven approaches such as the bootstrap \citep{tibshirani1993introduction} and subsampling \citep{politis1999subsampling}. Many tests based on these methods have been extensively studied (see, e.g., \cite{van1996weak, romano1988bootstrap, hu2025maxsliced}). However, as noted in \citet[Section~0.5.2]{albert2015tests}, these methods do not guarantee Type~I error control in a non-asymptotic sense, i.e., for small sample sizes — a property that is ensured by permutation-based approaches. As a result, when both methods are applicable and non-asymptotic Type I error control is required, the permutation approach should be preferred.

For the reasons discussed above, we adopt the permutation approach for the two-sample test based on the sliced Wasserstein distance.

\section{EXISTING APPROACHES FOR CONTROLLING THE RANDOM CRITICAL VALUE OF PERMUTATION-BASED (TWO-SAMPLE) TESTS}\label{sec:Technicalreview}

In the previous section, we motivated the use of a permutation approach for the sliced Wasserstein distance two-sample test. We now review how the challenges associated with the theoretical analysis of permutation based two-sample tests have been addressed in the literature (see also, for instance, \citet[Section~1.2]{kim2022minimax}).

A central difficulty in analyzing the non-asymptotic power of permutation tests lies in controlling the random critical value of the test, and in particular the dependence structure induced by permutation sampling. To the best of our knowledge, no prior work has addressed this issue in the context of the sliced Wasserstein distance.

In this section, we review two existing approaches for controlling the random rejection threshold in permutation-based tests. Both rely on the structure of U-statistics: the first through a coupling argument \citep{kim2022minimax}, and the second via a concentration inequality for permuted sums \citep{albert2019concentration}. We then explain why current techniques cannot be directly extended to the sliced Wasserstein setting, motivating the development of a new analytical framework---one of the main theoretical contributions of this work.

\subsection{Coupling technique for U-statistics}
\citet[Section~6]{kim2022minimax} investigate the performance of permutation-based procedures with a  focus on degenerate second-order U-statistics, a broad class of estimators which encompasses many commonly used two-sample test statistics \citep[Chapter~5]{serfling2009approximation}. This line of work has inspired several recent theoretical analysis in permutation-based two-sample testing \citep{schrab2023mmd, kim2023differentially, choi2024computational, chatalic2025efficient}. In order to clarify this approach, we begin by formally introducing  second-order U-statistics. We use the same setting and notation as in Section~\ref{permutation approach}.
\begin{definition}[Second order U-statistic]
Let $\mathsf{X}$ be a measurable space.
Let $g: \mathsf{X} \times \mathsf{X} \to \mathbb{R}$ be a measurable symmetric bivariate function, that is, $g(x,y)=g(y,x)$ for all $x,y \in \mathsf{X}$. We introduce
\begin{align*}
    h(y_1,y_2;z_1,z_2)\coloneqq g(y_1,y_2)+g(z_1,z_2)-g(y_1,z_2)-g(y_2,z_1).
\end{align*}
Let $\mathbf{i}^2_n$ be the set of all couples drawn without replacement from the set $\{1,\dots,n\}$. Then, the corresponding U-statistic is given by
\begin{align*}
U_{n,m}\coloneqq \dfrac{1}{n(n-1)m(m-1)}\sum_{(i_1,i_2)\in \textbf{i}^n_2}\sum_{(j_1,j_2)\in\textbf{i}^m_2}h(Y_{i_1},Y_{i_2};Z_{j_1},Z_{j_2}).
\end{align*}
Moreover, given a permutation $\pi \in S_N$ (with $N = n + m$), 
the permuted U-statistic associated with $\pi$ is defined as
\begin{align*}
U^{\pi}_{n,m}\coloneqq \dfrac{1}{n(n-1)m(m-1)}\sum_{(i_1,i_2)\in \textbf{i}^n_2}\sum_{(j_1,j_2)\in \textbf{i}^m_2}h(X_{\pi_{i_1}},X_{\pi_{i_2}};X_{\pi_{n+j_1}},X_{\pi_{n+j_2}}).   
\end{align*}
\end{definition}
Assume $n \le m$, and let $L = (l_1,\dots,l_n)$ denote an $n$-tuple drawn uniformly without replacement from the set $\{1,\dots,m\}$. 
Given $L$, we define the auxiliary test statistic
\begin{align*}
\tilde{U}^{\pi,L}_{n,m}
:= \frac{1}{n(n-1)}
\sum_{(k_1,k_2)\in\textbf{i}^{n}_2}
h \big(
X_{\pi_{k_1}}, X_{\pi_{k_2}};\,
X_{\pi_{n+l_{k_1}}}, X_{\pi_{n+l_{k_2}}}
\big).
\end{align*}
Note that the statistic $U^{\pi}_{n,m}$ 
is the conditional expectation of 
$\tilde{U}^{\pi,L}_{n,m}$ with respect to $L$,  given all other random quantities, that is,
\begin{equation}\label{eq:ustatisticconnection}
U^{\pi}_{n,m}
= \mathbb{E}_L\!\left[
\tilde{U}^{\pi,L}_{n,m}
\,\middle|\,
\mathcal{X}_{N}, \pi \right].
\end{equation}
The usefulness of this construction becomes apparent when applying the following lemma from \citet{kim2022minimax}. It is a coupling argument based on a symmetrization trick \citep{dumbgen1998symmetrization}.
\begin{lemma}[Coupling with i.i.d.\ random variables]\label{lem:coupling}
Let $N\ge 1$ and let $\pi$ be a random permutation uniformly distributed on the symmetric group $S_N$, i.e., the set of all permutations of $\{1,\dots,N\}$.
Let $k=\lfloor N/2\rfloor$ and let $\Delta=(\delta_1,\dots,\delta_k)$ be a random vector of independent Bernoulli trials, taking values in $\{0,1\}^k$ and independent of $\pi$. Define the transformation $T_\Delta:S_N\to S_N$ of any permutation $\tau=(\tau_1,\dots,\tau_N)\in S_N$ as,
\begin{align*}
\big(T_\Delta(\tau)\big)_{2i-1}=\delta_i\,\tau_{2i-1}+(1-\delta_i)\,\tau_{2i},\qquad
\big(T_\Delta(\tau)\big)_{2i}=(1-\delta_i)\,\tau_{2i-1}+\delta_i\,\tau_{2i},    
\end{align*}
for $i=1,\dots,k$.
Let $\pi'\coloneqq T_\Delta(\pi)$. Then $\pi'$ and $\pi$ are identically distributed.
\end{lemma}
For completeness we provide a proof of this lemma.
\begin{proof}
Assume that $N$ is even. If $N$ is odd, we simply set $\left(T_{\Delta}(\tau)\right)_N=\tau_N$ and the proof remains unchanged.

To prove that $\pi'$ and $\pi$ are identically distributed, it suffices to show that, for any arbitrary permutation $\sigma \in S_N$,
\begin{align*}
\mathbb{P}(\pi'=\sigma)=\mathbb{P}\left(\pi=\sigma\right).
\end{align*}

By applying the law of total probability, we obtain
\begin{align*}
\mathbb{P}\left(\pi'=\sigma\right)=\mathbb{E}\left[\mathbb{P}\left(T_{\Delta}(\pi)=\sigma \mid \Delta \right)\right].
\end{align*}

Since $T_{\Delta}$ is an involution, i.e., $T_\Delta(T_\Delta(\tau))=\tau$ for all $\tau \in S_N$, the event $\{ T_\Delta(\pi) = \sigma \}$ can be equivalently expressed as $\{ \pi = T_\Delta(\sigma) \}$. Moreover, since $\pi$ and $\Delta$ are independent, conditioning on $\Delta$ 
does not affect the distribution of $\pi$. Following these observations, it holds that
\begin{align*}
\mathbb{P}\left(\pi'=\sigma\right)=\mathbb{E}\left[\mathbb{P}\left(\pi =T_{\Delta}(\sigma) \mid \Delta \right)\right] =   \mathbb{E}\left[\mathbb{P}\left(\pi =T_{\Delta}(\sigma) \right)\right] =\mathbb{E}\left(\frac{1}{N!}\right)=\frac{1}{N!}=\mathbb{P}\left(\pi=\sigma\right).  
\end{align*}
Since this equality holds for every $\sigma \in S_N$, the claim follows.
\end{proof}

The above lemma implies that the distribution of $\tilde{U}^{\pi,L}_{n,m}$ remains unchanged if we randomly swap $X_{\pi_k}$ and $X_{\pi_{n+l_k}}$ for $k \in \{1,\dots,n\}$. 
In other words, it allows us to connect the statistic to i.i.d.\ Bernoulli random variables, which are easier to handle analytically. 
Moreover, due to the symmetry of $g(x,y)$ and the definition of $h$, this argument naturally extends to i.i.d.\ Rademacher random variables, viewed as random sign flips. 

Indeed, let $\zeta_1,\dots,\zeta_n$ be i.i.d.\ Rademacher random variables. Then
\begin{align} \label{eq:UandRademacher}
\tilde{U}^{\pi,L,\zeta}_{n,m}
\coloneqq \frac{1}{n(n-1)}
\sum_{(k_1,k_2)\in\textbf{i}^{n}_2}
\zeta_{k_1}\zeta_{k_2}\,
h\big(
X_{\pi_{k_1}}, X_{\pi_{k_2}};\,
X_{\pi_{n+\ell_{k_1}}}, X_{\pi_{n+\ell_{k_2}}}
\big) \stackrel{(d)}{=} \;\tilde{U}^{\pi,L}_{n,m}.
\end{align}
This machinery provides a way to bound the tail probability of $U^{\pi}_{n,m}$. Indeed, for any $\lambda > 0$ and $t > 0$, by applying a Chernoff bound \citep[Section~2.3]{vershynin2018high} and Jensen's inequality \citep[Section~1.6]{vershynin2018high} together with \eqref{eq:ustatisticconnection}, we obtain
\begin{align*}
\mathbb{P}_{\pi}\left(U^{\pi}_{n,m}>\lambda \middle| \mathcal{X}_{N}\right)\le e^{-\lambda t}\mathbb{E}_{\pi}\left[\exp\left(\lambda U^{\pi}_{n,m}\right)\middle|\mathcal{X}_N\right]\le e^{-\lambda t}\mathbb{E}_{\pi,L}\left[\exp\left(\lambda \tilde{U}^{\pi,L}_{n,m}\right)\middle|\mathcal{X}_N\right].
\end{align*}
Then, by \eqref{eq:UandRademacher}, it follows that
\begin{align*}
\mathbb{E}_{\pi,L}\left[\exp\left(\lambda \tilde{U}^{\pi,L}_{n,m}\right)\middle| \mathcal{X}_N\right]=\mathbb{E}_{\pi,L,\zeta}\left[\exp \left(\lambda\tilde{U}^{\pi,L,\zeta}_{n,m}\right) \middle| \mathcal{X}_N\right].
\end{align*}
As a result, 
\begin{align*}
\mathbb{P}_{\pi}\left(U^{\pi}_{n,m}>\lambda \mid \mathcal{X}_{N}\right) \le \mathbb{E}_{\pi,L,\zeta}\left[\exp \left(\lambda\tilde{U}^{\pi,L,\zeta}_{n,m}\right) \mid \mathcal{X}_N\right].
\end{align*}
The right-hand side, which involves Rademacher averages, is a well-studied quantity that can be controlled using standard decoupling arguments \cite[Chapter~6]{vershynin2018high}. This, in turn, provides an upper bound on the concentration of $U^{\pi}_{n,m}$ and allows control of the random critical value in the permutation U-statistic test.


Several works on MMD-based two-sample test use this approach by expressing their test statistic as the sum of a U-statistics with respect to the positive definite kernel $g$ associated to the MMD and a controllable remainder term. In contrast, the sliced Wasserstein distance does not relate naturally to U-statistics, preventing a direct application of this approach.

While \citet{kim2022minimax} rely on the symmetrization trick of \citet{dumbgen1998symmetrization} to address the dependence introduced by permutations, \citet{albert2019concentration} takes a different route and establish a concentration bound for permuted sums, offering an alternative way to control the random quantile term. We describe this second approach in the next section.

\subsection{Concentration of permuted sums}
Continuing our review of methods for controlling the random critical value in permutation-based tests, we now turn to the work of \citet{albert2019concentration}.
As shown in \citet[Section~2.2]{albert2015bootstrappermutationtestsindependence}, the test statistic considered in \cite{albert2019concentration} is a rescaled version of a U-statistic. Although their analysis focuses on independence testing, its contribution extends beyond this specific setting by introducing a distinct framework for analyzing random critical values and highlighting the key role of concentration inequalities in this task.

The main contribution of this work lies in Section~2.2, where the author establishes a concentration inequality for permuted sums in a general setting by leveraging fundamental inequalities for random permutations of \cite{talagrand1995concentration}. In particular, they present the following result
\begin{lemma}[Theorem~2.1 of \cite{albert2019concentration}]
Let $\{a_{i,j}\}_{1\le i,j\le N}$ be a collection of any real numbers, and $\pi$ be random uniform permutation in $S_N$. Consider $Z_N=\displaystyle \sum_{i=1}^{N}a_{i,\pi(i)}$. Then, for all $x>0$,
\begin{equation}\label{eq:concentrationpermutedsums}
\mathbb{P}\left(\left|Z_N-\mathbb{E}[Z_N]\right| \ge 2\sqrt{2\left(\frac{1}{n}\sum_{i,j=1}^{N}a^2_{i,j}\right)}+2\max_{1\le i,j\le N}\left\{\left|a_{i,j}\right|\right\}x\right)\le 16e^{1/16}\exp\left(-\dfrac{x}{16}\right).    
\end{equation}
\end{lemma}
In the final section, the author illustrates the use of this theorem by analyzing the non-asymptotic behavior of a permutation-based independence test introduced in \citet{albert2015bootstrappermutationtestsindependence}. Before going into the details of how they applied the theorem, we now recall in detail their problem of interest and the corresponding test statistic.

Let $\mathsf{X}$ be a separable space, and let $\mathcal{X}_N = (X_1, \ldots, X_N)$ be i.i.d.\ samples from a joint distribution $P$ on $\mathsf{X}^2$, 
where each $X_i = (X_i^{1}, X_i^{2})$ has marginals $P^{1}$ and $P^{2}$ corresponding to its coordinates. The goal is to test whether $P = P^{1}\otimes P^{2}$. 
To this end, they consider the following test statistic, whose motivation is detailed in \citet{albert2015bootstrappermutationtestsindependence}:
\begin{equation*}
T_{\delta}(\mathcal{X}_N)
= \frac{1}{N-1}\!\left(
\sum_{i=1}^{N} \varphi_{\delta}(X_i^{1},X_i^{2})
\;-\;
\frac{1}{n}\sum_{i=1}^{N}\sum_{j=1}^{N} \varphi_{\delta}(X_i^{1},X_j^{2})
\right).
\end{equation*}
where $\varphi_{\delta}$ is a measurable real-valued function on $\mathsf{X}^2$ potentially depending on some unknown
parameter $\delta$.

Then, for any random permutation $\pi$ uniformly distributed over $\{1, \dots, N\}$, 
the corresponding permuted sample is defined as
\begin{align*}
\mathcal{X}^{\pi}_N = (X^{\pi}_1, \dots, X^{\pi}_N),
\qquad 
\text{where } X^{\pi}_i = (X^1_i, X^2_{\pi(i)}), \quad \forall\, 1 \le i \le N,
\end{align*}
and the associated permuted test statistic is given by
\begin{equation}\label{eq:permutedT}
T_{\delta}(\mathcal{X}^{\pi}_N)
= \frac{1}{N-1}\!\left(
\sum_{i=1}^{N} \varphi_{\delta}(X_i^{1}, X_{\pi(i)}^{2})
\;-\;
\frac{1}{n}\sum_{i=1}^{N}\sum_{j=1}^{N} \varphi_{\delta}(X_i^{1}, X_j^{2})
\right).
\end{equation}
To apply their concentration bound, they express the permuted test statistic as the difference between a random variable and its expectation. In particular, by introducing $\tilde{{Z}}(\mathcal{X}_N)=\sum\limits^N_{i=1}\varphi_\delta(X^1_i,X^2_{\pi(i)})$, the permuted test statistic defined in \eqref{eq:permutedT} can be rewritten as
\begin{equation*}
T_{\delta}(\mathcal{X}^{\pi}_N)
= \frac{1}{N-1}\!\left(
\tilde{Z}(\mathcal{X}_N)-\mathbb{E}\left[\tilde{Z}(\mathcal{X}_N) \mid \mathcal{X}_N\right]
\right).
\end{equation*}
This reformulation aligns the left-hand side of the probability inequality in \eqref{eq:concentrationpermutedsums} with the permuted test statistic, thereby allowing the authors to directly apply the concentration result to derive an upper bound on its quantile. However, this approach cannot be extended to the case of the sliced Wasserstein distance, since the latter cannot be written as such a centered difference.

From the preceding technical review concerning the control of the conditional quantile of a permuted test statistic, two main insights can be drawn:
\begin{itemize}
\item Concentration inequalities are a powerful tool for analyzing conditional quantiles, as they provide sharp non-asymptotic high-probability bounds—typically with exponentially small tails—that lead to the desired logarithmic dependence on the Type~I error level~$\alpha$, as discussed in \citet[Section~3.2]{albert2019concentration}. 
\item Unlike $U$-statistics, which are centered at zero under permutation, or the statistic in \citet{albert2015bootstrappermutationtestsindependence}, which can be written as a deviation from its expectation, analyzing the permuted sliced Wasserstein statistic requires explicit control of its expectation under permutation of the samples. We address this challenge in Section~\ref{Sec:OptimalTransportResult}.
\end{itemize}
\section{GAUSSIAN MEAN SHIFT EXPERIMENT}\label{sec:addtionalexperiment}
Due to page constraints in the main paper, we present an additional experiment here. In this setting, samples are drawn from $\mu = \mathcal{N}(0, I_{60})$ and $\nu = \mathcal{N}(\mathbf{m}, I_{60})$, where $\mathbf{m} = (0.6, 0.6, 0, \dots, 0) \in \mathbb{R}^{60}$ is a 60-dimensional vector whose first two entries are equal to $0.6$.

Following the experimental setup from Section~\ref{sec:Testsonsynthetic}, we compare the sliced Wasserstein (SW) tests with the Projected Wasserstein (PW) test \citep{wang2021linear} and the MMD test \citep{gretton2012mmd} using linear, Gaussian and Laplace kernels. As observed, the MMD test with a linear kernel shows a significant improvement in this setting. This improvement can be explained by the fact that for the linear kernel $k(x, y) = x^\top y$, the MMD test reduces to a mean-difference test, making it particularly well-suited for detecting to mean shifts.

It is also observed that, in this scenario, the SW-based tests exhibit lower performance compared to the three MMD-based tests. Nevertheless, the statistical power of SW tests can be considerably improved by increasing the number of projections. This observation motivates the extended investigation presented in Section~\ref{Sec:EffectOfTheNumberProjections}.
\begin{figure}[h]
  \centering
\includegraphics[width=0.5\linewidth]{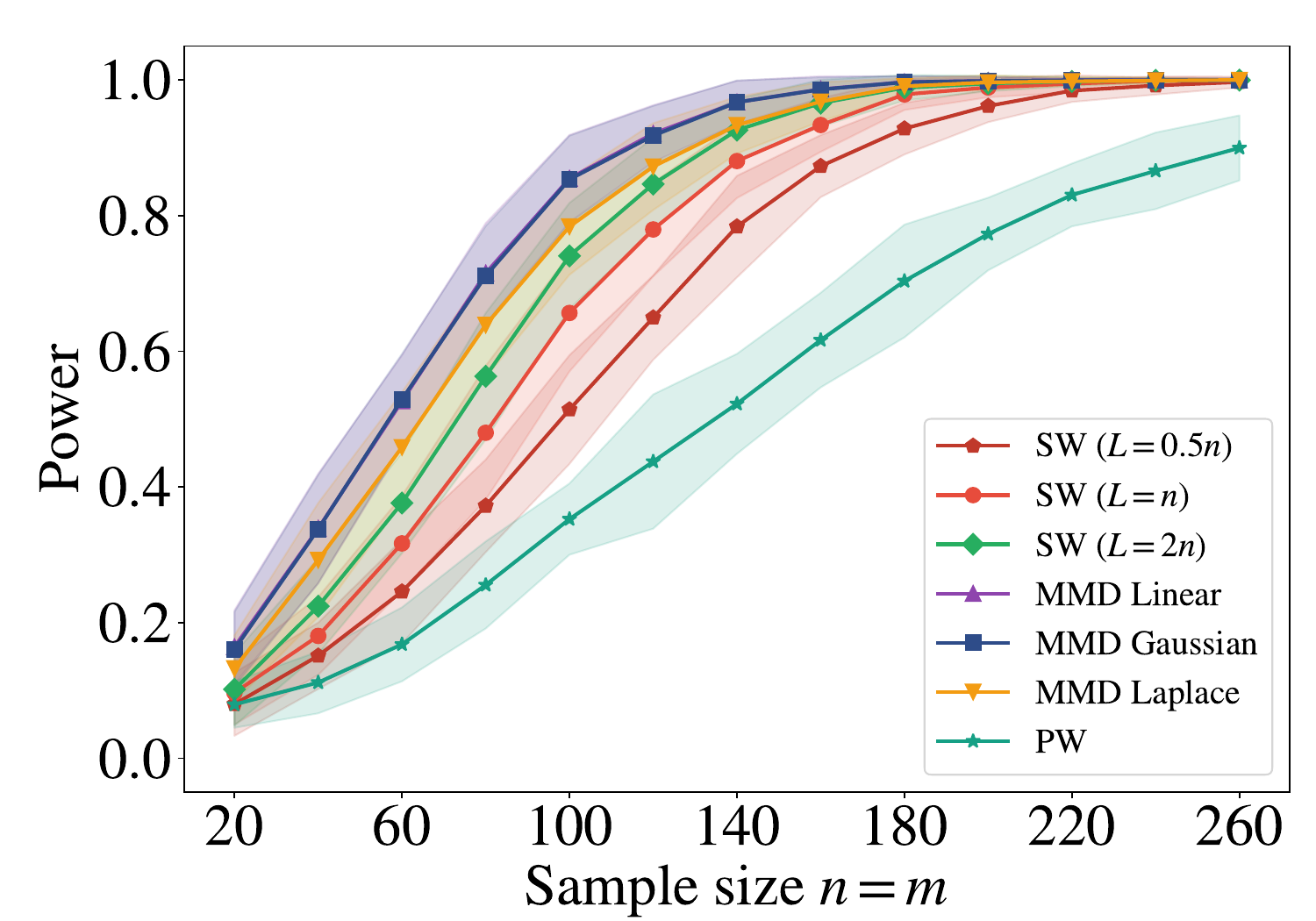}
  \caption{Power vs. number of sample size: Gaussian mean shift scenario}
\end{figure}